\documentclass[final,12pt]{clear2026}
\usepackage{macros/preamble}
\usepackage{macros/influence-diagrams}
\usepackage{times}

\begin{document}

\title[Causal Foundations of Collective Agency]{Causal Foundations of Collective Agency}

\clearauthor{%
 \Name{Frederik Hytting Jørgensen} \Email{frederik.hytting@math.ku.dk}\\
 \addr Copenhagen Causality Lab and Pioneer Centre for AI,\\
 University of Copenhagen, Denmark
 \AND
 \Name{Sebastian Weichwald} \Email{sweichwald@math.ku.dk}\\
 \addr Copenhagen Causality Lab and Pioneer Centre for AI,\\
 University of Copenhagen, Denmark
 \AND
 \Name{Lewis Hammond} \Email{lewis.hammond@cooperativeai.org}\\
 \addr Cooperative AI Foundation, United Kingdom
}

\maketitle

\begin{abstract}
    A key challenge for the safety of advanced AI systems is the possibility that multiple 
    simpler agents might inadvertently form a collective %
    agent with capabilities and goals distinct from those of any individual.
    More generally, determining when a group of agents %
     can be viewed as a unified collective agent is a foundational question 
     in the study of interactions and incentives in both biological and artificial systems.
    We adopt a behavioral perspective in answering this question, ascribing collective agency to a group when viewing the group's joint actions as rational and goal-directed successfully predicts its behavior.
    We formalize this perspective on collective agency 
    using causal games 
    \citep{hammond2023reasoning} -- which are causal models of strategic, multi-agent interactions -- 
    and causal abstraction \citep{rubenstein2017causal,beckers2019abstracting} -- which formalizes when 
    a simple, high-level model faithfully captures a more complex, low-level model.
    We use this framework to solve a puzzle regarding multi-agent incentives in actor-critic models 
    and to make quantitative assessments of the degree of collective agency exhibited by different voting mechanisms.
    Our framework aims to provide a foundation for theoretical and empirical work to understand, predict, 
    and control emergent collective
    agents in multi-agent 
    AI systems.
\end{abstract}


\section{Introduction}
\label{sec:intro}

When does it make sense to view a group of individual agents as a unified collective agent?
While this question may seem rather subjective at first, there is an intuitive sense in which we might ascribe collective agency to, for example, a well-organized, efficient, and synchronized team, as opposed to a random assortment of individuals pursuing independent courses of action.
This is despite the fact that such a team may not have a leader or hierarchical structure (such as a jazz ensemble improvising, or a school of fish avoiding a predator), and may not even have the same objectives as one another (such as a market of buyers and sellers, or pollinating insects and flowering plants).

Indeed, humans often use notions of collective agency to understand the world \citep{Roth2017,Schweikard2021}.
For example, it is not uncommon for people to use language such as ``country A wants X, which conflicts with the interests of country B'' or ``company C is pursuing a new strategy Y''. 
While statements like these are usually supposed to be approximations, 
they can still be helpful for understanding the world and making predictions about the future.
Some have even gone so far as to argue that individual minds are composed of agentic subsystems \citep{Nagel1971,Minsky1988},
in which case there is no privileged level of abstraction on which to locate agency. 
How can we make sense of agency arising at various levels of abstraction?

Far from being purely philosophical, this question has important implications for understanding and predicting collective behavior in multi-agent systems, and hence for ensuring the safety of networks of artificial agents.
A key safety challenge here is the possibility that multiple simpler AI systems might inadvertently form a `super-agent' with capabilities and goals distinct from any individual in the group \citep{drexler2019reframing,hammond2025multi}.
Compared to AI tools, the ability of artificial \emph{agents} to make plans and take actions in pursuit of complex goals makes them not only more useful, but also more harmful if those goals are misaligned or their pursuit produces negative side effects \citep{chan2023harms}.
For example, competitive pressures may lead individually rational AI agents to rapidly exhaust collective resources \citep{Piatti2024}, or a group of agents might combine their harmless individual capabilities to override safeguards and execute a cyberattack \citep{Jones2024}.

\subsection{Example: Actor-Critic Agents}
\label{ex: actor-critic introduction}

As a simple example of a collective agent, consider the well-known family of 
reinforcement learning (RL) agents known as \textit{actor-critic} (AC) algorithms \citep{Konda2000}.
An AC agent comprises a critic that attempts to quantify how well the actor is performing relative to an objective, and an actor that attempts to improve its strategy according to the critic's judgment.
A single-step decision problem featuring an AC algorithm is shown as a causal game \citep{hammond2023reasoning}
in \Cref{fig:AC_algo_intro}. %
\begin{figure}[h]
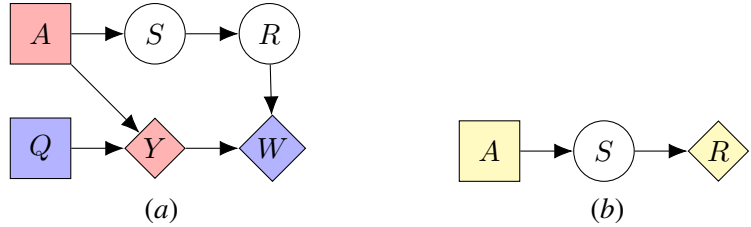

    \centering
    \subfigure[]{%
        \begin{influence-diagram}
            \node (A) [decision, player1] {$A$};
            \node (S) [right = of A] {$S$};
            \node (R) [right = of S] {$R$};
            \node (Q) [below = of A, decision, player2] {$Q$};
            \node (Y) [right = of Q, utility, player1] {$Y$};
            \node (L) [right = of Y, utility, player2] {$W$};
            \edge {A} {S,Y};
            \edge {S} {R};
            \edge {Q} {Y};
            \edge {Y,R} {L};
        \end{influence-diagram}
    }
    \hspace{4em}
    \subfigure[]{%
        \begin{influence-diagram}
            \node (A) [decision, player3] {$A$};
            \node (S) [right = of A] {$S$};
            \node (R) [right = of S, utility, player3] {$R$};
            \edge {A} {S};
            \edge {S} {R};
        \end{influence-diagram}
    }
    \caption[]{(a) A causal game representing an AC algorithm.
        In a causal game, square nodes represent decisions, circular nodes represent 
        chance variables, diamond nodes represent utilities, and edges  represent causal dependencies.
        In this game, the actor chooses an action $a$, which leads 
        to a new state $s$ and in turn the (true) reward $r$.%
        \footnotemark{}
        The critic, on the other hand, chooses a Q function
        $q$ that takes as input the action $a$ to produce a \textit{predicted} reward $y = q(a)$.
        The actor's utility is given by $y$ and the
         critic's utility is given by the loss $w = \ell(y, r)$ where $\ell$ is a loss function.
		(b) A causal influence diagram (a single-agent causal game) 
        forming a valid abstraction of the causal game in panel 
        (a) of \Cref{fig:AC_algo_intro}. In this game, the (collective) 
        AC agent effectively acts so as to maximize the reward directly.
		Both figures are adapted from \citet{kenton2023discovering}.}
    \label{fig:AC_algo_intro}
\end{figure}

\footnotetext{Note that in this simple, single-step AC problem, we leave the (fixed) starting state as implicit,
such that $Y$ only has $A$ as a parent. While it is common for the reward function to
 be a function of the current state, action, 
 and the new state, including sufficient information 
 within the new state leads to an equivalent formulation 
 where the reward is simply a function of the variable $S$. 
 }

Even though neither the actor nor the critic is maximizing for the reward, 
the overall system still intuitively seems to pursue this goal.
In this work, we argue that this is precisely because there is a valid \emph{abstraction} that  faithfully represents the low-level causal game \citep{rubenstein2017causal,beckers2019abstracting,beckers2020approximate}, and which corresponds to the overall agent (shown in panel (b) of \Cref{fig:AC_algo_intro}).
While \citet{kenton2023discovering} also suggested that panel (b) of \Cref{fig:AC_algo_intro} was a possible alternative model of the true AC dynamics in panel (a), they made no argument as to why.
We solve this problem, providing a formal account of agency at different levels of abstraction.

\subsection{Related Work}
\label{sec:related_work}

The question of when multiple individuals constitute a unified collective agent has long occupied philosophers, with foundational work establishing the importance of collective intentionality and shared intentions \citep{Searle1990,Tuomela2006,Ludwig2007,List2011,Pacherie2013,Bratman2014}.
While most of these prior works focus on a rich notion of collective agency, such as requiring the mental modeling of collaborators, we take a behavioral or `black-box' approach,
reducing the question of whether something is a collective agent to whether ascribing collective agency is predictive of the group's behavior.
Though these works provide conceptual foundations, they generally lack the formal machinery needed to rigorously identify collective agents.
Recent work has made progress on defining \emph{individual}
 agents using causal and decision-theoretic tools
  \citep{Orseau2018,kenton2023discovering,macdermott2024measuring,Xu2024,abel2025agency,Everitt2025,rajcic2025goal},
   but none formally consider \emph{collective} agents.

For multi-agent settings, \citet{hammond2023reasoning} introduce causal games -- a framework generalizing multi-agent influence diagrams \citep{Koller2003} to higher levels of Pearl's causal hierarchy \citep{Pearl2009a}.
Other game-theoretic approaches provide models of coalition formation \citep{Ray2007,Elkind2016} but do not address whether such coalitions constitute unified agents.
The concept of causal abstraction \citep{rubenstein2017causal,beckers2019abstracting,beckers2020approximate} formalizes when a high-level causal model is a valid abstraction of a detailed low-level model, preserving causal relationships when variables are aggregated.
However, existing causal abstraction work focuses on 
causal models without agents, requiring extension to multi-agent settings with strategic 
interactions.
A discussion of further related work can be found in \Cref{app:related_work}.

\subsection{Contributions}

After reviewing the necessary background in \Cref{sec:background}, we first use causal games \citep{hammond2023reasoning} to define what we mean by an individual decision-making agent, according to a given notion of rationality (such as playing a best response), in \Cref{sec:agency}.
Then, in \Cref{sec:collective_agency} we leverage the concept of causal abstraction to provide a principled account of when a set of `low-level' agents can be usefully modeled as a single `high-level' collective agent.
Using these formal tools, we prove a proposition that rules out the emergence of non-trivial agency at higher levels of abstraction and analyze
the emergence of collective agency in the actor-critic example (\Cref{ex: actor-critic introduction}).
We further illustrate this framework in \Cref{sec:surrogate_modeling} by empirically analyzing different voting mechanisms and the extent to which they transform the different voters into a collective agent.
\Cref{sec:conclusion} concludes with a brief summary and an overview of directions for future work.
\Cref{app:applications} further discusses potential applications of our framework
 to multi-agent reinforcement learning and large language model agents.

\section{Background}\label{sec:background}

Our notation and setup draw inspiration from \citet{geiger2025causal}. We begin with the 
notion of a signature, which describes a set of variables and the values they may take.

\begin{definition}
    A \textbf{signature} is a tuple $\Sigma=(\rX,(\dom(X_i))_{i\in [d]})$, where $\rX$ is a 
    tuple of variables $(X_1,\dots, X_d)$ and $(\dom(X_i))_{i\in [d]}$ is a tuple of ranges for 
    each variable.\footnote{We sometimes consider $\rX$ a set rather than a tuple.} 
\end{definition}

For any subset $\rY\subseteq \rX$, we define $\dom(\rY)=\bigtimes_{Y\in \rY}\dom(Y)$.
Formally, we assume that the variables have disjoint ranges.\footnote{The assumption of 
disjoint ranges can be made without loss of generality: we 
    subscript every value with the variable name to ensure uniqueness, such that, for example, 
    $x_{X_i}\neq x_{X_j}$ for $x\in \mathbb{R}$ and $i\neq j$. If it is clear from the context 
    which range a value belongs to, we may omit the subscript.}
Therefore, we may consider the elements of $\dom(\rY)$ as sets rather than tuples. This technical assumption allows us to identify settings of
variables $\rY\subseteq \rX$ with sets $\ry\in \dom(\rY)$, which is useful for defining projections 
 (\definitionref{def: projection}) and
  mechanized abstractions (\definitionref{def: collective agency}).

\begin{definition} 
    \label{def: projection} 
Given two sets of variables $\rY,\rZ \subseteq \rX$, and setting $\ry\in \dom(\rY)$, we define 
the \textbf{projection}
    $\proj_{\rZ}(\ry)=\ry \cap \left(\bigcup_{Z\in \rZ} \dom(Z)\right)$.
\end{definition}

For example, consider a signature given by the real-valued variables $\rX=\{X_1,X_2\}$,
i.e., where each has range $\{x_{X_i}\mid x\in \mathbb{R}\}$ for $i \in {1,2}$.
Then $\proj_{\{X_1\}}(\{4_{X_1},5_{X_2}\})=\{4_{X_1}\}$ and $\proj_{\{X_1\}}(\{5_{X_2}\})=\emptyset$.

\begin{definition} \label{def: det SCM}
    Given a signature $\Sigma$, a \textbf{deterministic (cyclic) structural causal model} 
    $\widetilde{\mathcal{M}}$ over $\Sigma$ is a set of functions $\{\mathcal{F}_X\}_{X\in \rX}
    $, where $\mathcal{F}_X: \dom(\rX\backslash\{X\})\to \dom(X)$. By $\Sol(\widetilde{\mathcal{M}})$ we denote the 
    elements $\rx\in \dom(\rX)$ -- i.e. the \textbf{solutions} -- that satisfy $\mathcal{F}_X(\rx\backslash\proj_X(\rx))=\proj_X(\rx)$ for all $X\in \rX$.
\end{definition}

Notice that while deterministic structural causal models may be cyclic, we will assume that probabilistic causal models are acyclic (see \definitionref{def: SCM}).

\begin{definition}
    A \textbf{hard intervention} in a deterministic SCM is a setting $\ry$ of some subset of variables 
    $\rY\subseteq \rX$. By $\Sol(\widetilde{\mathcal{M}};\ry)$ -- the \textbf{solutions} in $\widetilde{\mathcal{M}}$, given intervention $\ry$ --
     we denote the set of values $\rx\in \dom(\rX)$ 
    such that $\mathcal{F}_X(\rx\backslash\proj_X(\rx))=\proj_X(\rx)$ 
    for $X\notin \rY$ and 
    $\proj_X(\rx)=\proj_X(\ry)$ for $X\in \rY$. 
\end{definition}

A hard intervention corresponds to replacing the functions $\mathcal{F}_X$ with $\rz \mapsto 
\proj_X(\ry)$ for $X\in \rY$, and the solutions $\Sol(\widetilde{\mathcal{M}};\ry)$ correspond 
to the settings $\rx$ of variables $\rX$ that are solutions to this new system of equations.

\begin{definition}
    \label{def: SCM}
    Given a signature $\Sigma$, 
    a \textbf{structural causal model} (SCM) $\mathcal{M}$ over 
    $\Sigma$ is a tuple
    $$\left((\rV,\bm{\mathcal{E}}), \mathcal{G},\{\mathcal{F}_{V}\}_{V\in \rV},\mathbb{P}(\bm{\mathcal{E}})\right),$$
    where \begin{enumerate}
        \p $(\rV,\bm{\mathcal{E}})$ is a partition of the variables in $\Sigma$ into noise variables $\bm{\mathcal{E}}$ and object variables $\rV$. We assume that each object variable has a unique noise variable associated with it, that is, $\bm{\mathcal{E}}=\{\mathcal{E}_V\}_{V\in \rV}$.  
        \p $\mathcal{G}$ is a directed acyclic graph over the object variables $\rV$. 
        \p $\{\mathcal{F}_V\}_{V\in \rV}$ is a set of structural assignments $\mathcal{F}_V: \dom(\PA_V^{\mathcal{G}}\cup \{\mathcal{E}_V\})\to \dom(V)$, where $\PA_V^{\mathcal{G}}$ denotes the parents of $V$ in $\mathcal{G}$.
        \p $\mathbb{P}(\bm{\mathcal{E}})$ is a distribution on $\dom(\bm{\mathcal{E}})$. We assume that the variables $\{\mathcal{E}_V\}_{V\in \rV}$ are jointly independent.
    \end{enumerate}
    An SCM $\mathcal{M}$ induces a unique distribution $\mathbb{P}_\mathcal{M}(\rV)$ over the 
    object variables $\rV$ such that $V=\mathcal{F}_V(\PA_V,\mathcal{E}_V)$, in distribution, 
    for every $V\in \rV$. Since the noise variables are jointly independent, the induced 
    distribution $\mathbb{P}_\mathcal{M}(\rV)$ is \textit{Markovian} with respect to $\mathcal
    {G}$ \citep[][Prop. 6.31]{Peters2017}.
\end{definition}

In order to consider causal games in which the mechanisms (i.e. the  structural functions and noise variables) governing an object variable can change if an agent selects a different strategy, it will be useful to define parameterized SCMs, which is a class of SCMs with functions that are indexed by parameters.\footnote{This restriction to parameterized functions -- again departing from \citep{hammond2023reasoning} -- is more natural for AI agents whose strategies are represented by parametric models such as a neural network.}

\begin{definition}
    Given a signature $\Sigma$, a \textbf{parameterized SCM} $\mathcal{M}^{\bm{\Theta}}$ over  
    $\Sigma$ is a tuple 
    $$\left((\rV,\bm{\mathcal{E}}), \mathcal{G},\{\Theta_V\}_{V\in \rV},
    \{\mathcal{F}^\theta_V\}_{V\in \rV,\theta\in \Theta_{V}},\mathbb{P}(\bm{\mathcal{E}})\right)$$
    where $(\rV,\bm{\mathcal{E}}), \mathcal{G}$, and $\mathbb{P}(\bm{\mathcal{E}})$ are as 
    for an SCM (\definitionref{def: SCM}), and
    \begin{enumerate}
        \p $\{\Theta_V\}_{V\in \rV}$ is a set of parameter spaces, one for each object  variable. 
        \p $\{\mathcal{F}^\theta_V\}_{V\in \rV,\theta\in \Theta_{V}}$ is a set of structural 
        assignments $\mathcal{F}^\theta_V: \dom(\PA_V^\mathcal{G}\cup \{\mathcal{E}_V\})\to \dom
        (V)$, which are indexed by variables $V$ and parameters $\theta_V \in \Theta_V$.
    \end{enumerate}
    For a given setting $\bm{\theta}=\{\theta_V\}_{V\in \rV}$ of the parameters $\theta_V\in 
    \Theta_V$, a parameterized SCM induces an SCM $\mathcal{M}^{\bm{\theta}}$.
\end{definition}

We are now ready to provide a formal definition of \emph{mechanized} structural causal models 
\citep{hammond2023reasoning}, in which the parameters $\theta_V$ 
of each structural function (representing the causal mechanisms in the model) are explicitly 
represented and governed by a separate deterministic cyclic SCM.\footnote{\citet{dawid2002influence} 
provides an earlier example of mechanized SCMs, 
though he does not consider the mechanism nodes themselves as being governed by an SCM, 
rather each mechanism node is parentless.} 
We depart slightly from \citet{hammond2023reasoning} by 
having the mechanisms being governed by functions rather than relations. 

\begin{definition}
A \textbf{mechanized SCM} $\mec$ is a tuple ($\widetilde{\mathcal{M}}$, $\mathcal{M}^{\bm{\Theta}}$) where 
\begin{enumerate}
    \p $\widetilde{\mathcal{M}}$ is a deterministic cyclic SCM with mechanism variables $\widetilde{\rV}$. 
    \p $\mathcal{M}^{\bm{\Theta}}$ is a parameterized SCM with object variables $\rV$. 
\end{enumerate}
We refer to $\rV$ as the \textbf{object nodes} and to $\widetilde{\rV}$ as the \textbf{mechanism nodes}. 
We assume that there is a one to one correspondence between the mechanism variables $\widetilde{\rV}$ and object variables $\rV$, and that $\mathcal{M}^{\bm{\Theta}}$'s parameter sets are $\Theta_V=\dom(\widetilde{V})$ for each $V\in \rV$. 
Every solution $\rs \in \Sol(\widetilde{\mathcal{M}})$ induces an SCM $\mathcal{M}^{\rs}$ with structural assignments 
$\{\mathcal{F}_V^{\proj_{\widetilde{V}}(\rs)} \mid V\in \rV\}$ and distribution $\mathbb{P}_{\rs}(\rV)$.  
By $\mathbb{P}_{\Sol(\widetilde{\mathcal{M}})}(\rV)$ we denote the set of distributions $\{\mathbb{P}_{\rs}(\rV)\mid \rs\in \Sol(\widetilde{\mathcal{M}})\}$.  
\end{definition}

To help the reader understand this formalism, we provide a worked out formalization of the `Battle of the Sexes' game as a mechanized SCM in \Cref{app: battle of the sexes}.

\section{Agency}
\label{sec:agency}

We define utilities as functions of nodes (i.e. object variables) rather than nodes in the causal graph itself, highlighting the constitutive rather than causal relationship between utilities and object variables \citep{xia2024neural}. 
For example, we may have two different agents, who both care about the number of bicycles in the world, but one agent wants to maximize the number of bicycles in the world, while the other agent wants to minimize the number of bicycles in the world. 
In that case, it would be problematic to have two different utility 
nodes because there would be a logical rather than causal relationship between the two utility nodes, 
which would violate the assumption of independent causal mechanisms \citep{Peters2017,jorgensen2025causal}. 

\begin{definition}
    \label{def: utility}
    Given a mechanized SCM $\mec$ with object variables $\rV$,
     we define a \textbf{utility function} as a function $\mathcal{U}: \dom(\rV)\to \mathbb{R}$.
\end{definition}

While we technically allow that a utility function $\mathcal{U}$ may depend on several object variables $\mathbf{U} \subseteq \mathbf{V}$, we will 
usually consider utilities that depend on a single object variable $U$. 
Given a utility function, we may wonder if any object level variables can be viewed as  
optimizing it.
For this purpose, we use rationality relations \citep{hammond2023reasoning}, which describe how an agent would react in a given context if employing a certain conception of rationality and having a certain utility function.

\begin{definition} 
Given a mechanized SCM $\mec$, a mechanism node $\widetilde{S}\in \widetilde{\rV}$ and a utility 
function $\mathcal{U}:\dom(\rV)\to \mathbb{R}$, a 
 \textbf{rationality relation} is a total relation $\mathcal{R}_{\widetilde{S}}
\subseteq \dom(\widetilde{\rV}\backslash\{\widetilde{S}\}) \times \dom (\widetilde{S})$.
If $\mathcal{F}_{\widetilde{S}}(\widetilde{\rc}) \in \mathcal{R}_{\widetilde{S}}(\widetilde{\rc})$
 \footnote{By $\mathcal{R}_{\widetilde{S}}(\widetilde{\rc})$ we denote the set 
 $\{\widetilde{s}\in \dom(\widetilde{S})\mid (\widetilde{\rc},\widetilde{s})\in \mathcal{R}_{\widetilde{S}}\}$.}
for every context $\widetilde
{\rc}\in \dom(\widetilde{\rV}\backslash\{\widetilde{S}\})$, then we say that $\widetilde{S}$ 
responds $\mathcal{R}$-rationally to utility $\mathcal{U}$.
\end{definition}

Intuitively, $\mathcal{R}_{\widetilde{S}}(\widetilde{\rc})$ specifies the set of policies an agent with 
decision node $S$ employing rationality relation $\mathcal{R}_{\widetilde{S}}$ could employ in context 
$\widetilde{\rc}$ if having utility function $\mathcal{U}$.

\begin{definition}
    \label{def: best response rationality}
Given a mechanism node $\widetilde{S}\in \widetilde{\rV}$ and a utility function $\mathcal{U}:\dom(\rV)\to \mathbb{R}$,
we define the \textbf{best response rationality relation} $\mathcal{R}^{\text{BR}}_{\widetilde{S}}$
 under utility $\mathcal{U}$ by the condition that 
\begin{align*}
    \widetilde{s}\in \mathcal{R}^{\text{BR}}_{\widetilde{S}}(\widetilde{\rc})
     \quad \text{ if and only if }  \quad
     \widetilde{s} \in \underset{\widetilde{s}'\in \dom(\widetilde{S})}{\argmax} \ \mathbb{E}_{\widetilde{s}',
     \widetilde{\rc}}(\mathcal{U}(\rV)) 
\end{align*}    
for every $\widetilde{\rc}\in \dom (\widetilde{\rV}\backslash\{\widetilde{S}\})$.
\end{definition}

 Best response rationality is the main rationality concept we consider. Sometimes, other rationality concepts,
such as subgame perfection, may be more appropriate \citep{hammond2023reasoning}. Different 
notions of rationality may be 
particularly relevant in
the context of collective agency for AI agents, where 
agents may be
interacting with copies of themselves or have 
access to each other's parameters \citep{critch2022cooperative, oesterheld2023similarity}.
 We leave a detailed investigation of various rationality relations for future work, but 
 in \Cref{app: first mover rationality}, we show that sometimes we do not have collective agency under best response 
 rationality, while we do have it for other rationality concepts.

Below we operationalize agency as rational, goal-driven behavior. 
While we do not claim that this is all that there is to the philosophical concept of agency, we do believe that this provides
a useful starting point for a formalization of agency that works across levels of abstraction. 

\begin{definition}
Given a mechanized SCM $\mec$, we say a mechanism variable $\widetilde{S}\in \widetilde{\rV}$ is an ($\mathcal{R}, \mathcal{U}$)-\textbf{agent}
 if $\widetilde{S}$ responds $\mathcal{R}$-rationally to utility $\mathcal{U}$.  
\end{definition}

\Citet{kenton2023discovering} discuss 
assumptions that make it possible to identify which nodes correspond to best-response rational agents and which nodes 
enter into their utility functions. 
Notice that every object variable is an agent with respect to best response rationality and a constant utility function. 
This is related to a classical insight in inverse reinforcement learning \citep{ng2000algorithms,skalse2024partial}, namely 
that reward functions are not generally identifiable from behavior. Likewise, if we allow any utility function, then 
any mechanism variable can be viewed as an agent. While modeling something as an agent optimizing for a constant utility function
does not make false predictions, it also does not make any non-vacuous predictions. 

\begin{definition}
    We say that a mechanism variable $\widetilde{S}\in \widetilde{\rV}$ is a \textbf{non-trivial ($\mathcal{R}, \mathcal{U}$)-agent} if it is an 
    ($\mathcal{R}, \mathcal{U}$)-agent and there exist two settings $\widetilde{\rc},\widetilde{\rc}'
    \in \dom(\widetilde{\rV}\backslash\{\widetilde{S}\})$ such that 
    $\mathbb{P}_{\widetilde{\rc},\mathcal{F}_{\widetilde{S}}(\widetilde{\rc})}(S\mid \PA_S)\neq 
    \mathbb{P}_{\widetilde{\rc}',\mathcal{F}_{\widetilde{S}}(\widetilde{\rc}')}(S\mid \PA_S)$
\end{definition}

\section{Collective Agency}
\label{sec:collective_agency}

\subsection{Mechanized Abstractions}

In this section, we refer to objects related to the high-level model using superscripts *. 
Given a low-level mechanized SCM $\mec$, 
and a corresponding high-level mechanized SCM $\mec^*$, we need to specify which of 
the low-level object variables $\rV$ 
correspond to which high-level object variables $V^*$. We call this correspondence a variable 
alignment. 

\begin{definition} 
Given two mechanized causal graphs $\mec$ and $\mec^*$, a \textbf{variable alignment} is a mapping $\part: \rV^*\to \mathcal{P}(\rV)\backslash\{\emptyset\}$ 
such that $\part_{V^*_1}\cap \part_{V^*_2}=\emptyset$ for any two variables $V_1^*,V_2^*\in \rV^*$ (where $\part_{V^*_1}:=\part(V^*_1)$). 
Given a variable alignment,
 we define $\part_{\widetilde{V}^*}\subseteq \widetilde{\rV}$ as those mechanism nodes 
 that correspond to the object variables in $\part_{V^*}$. 
\end{definition}

Given a variable alignment $\part$,
 a set of functions $\{\tau_{V^*}\}_{V^*\in \rV^*}$, where
  $\tau_{V^*}: \dom(\part_{V^*})\to \dom(V^*)$, induces 
  a function $\tau: \dom(\rV)\to \dom(\rV^*)$ given 
  by $\tau(\rv)=\bigcup_{V^*\in \rV^*} \tau_{V^*}(\proj_{\part_{V^*}}(\rv))$. We call 
  these functions \textbf{value mappings}.
  These functions determine
  how we translate values of low-level 
  object variables to values of high-level object variables.
In addition, we need to specify how we translate interventions 
on the low-level mechanisms variables to interventions
on the high-level mechanisms variables. These \textbf{intervention mappings} are given by partial functions 
$\{\omega_{\widetilde{V}^*}\}_{\widetilde{V}^*\in \widetilde{\rV}^*}$,
where $\omega_{\widetilde{V}^*}: \dom(\part_{\widetilde{V}^*})\rightharpoonup \dom(\widetilde{V}^*)$.
By $\text{dom}(\omega_{\widetilde{V}^*})$, 
we denote the subset of $\dom(\part_{\widetilde{V}^*})$ on 
which $\omega_{\widetilde{V}^*}$ is defined. For a subset 
$\widetilde{\rY}^*\subseteq \widetilde{\rV}^*$,
we define
$\text{dom}(\omega_{\widetilde{\rY}^*})=
    \bigtimes_{\widetilde{V}^*\in \widetilde{\rY}^*} \text{dom}(\omega_{\widetilde{V}^*})$.
The set of intervention mappings naturally induces a 
partial function $\omega$, from any setting of a subset of low-level collections 
$\bigcup_{\widetilde{V}^*\in \widetilde{\rY}^*} \part_{\widetilde{V}^*}$
 to a setting 
of $\widetilde{\rY}^*$, in particular, 
 given $\widetilde{\ry}\in \text{dom}(\omega_{\widetilde{\rY}^*})$, $\omega$ 
maps to $\bigcup_{\widetilde{V}^*\in \widetilde{\rY}^*} 
\omega_{\widetilde{V}^*}(\proj_{\part_{\widetilde{V}^*}}(\widetilde{\ry}))$.

\begin{definition}\label{def: collective agency}
    \label{def:constructive_abstraction}
    Assume that we have two mechanized causal 
    models $\mec$ and $\mec^*$.
    Let a node alignment $\part$ be given.
    We say that $\mec^*$ 
    is a \textbf{mechanized abstraction} of $\mec$ 
    under value 
    mappings $\{\tau_{V^*}\}_{V^*\in \rV^*} $ ($\tau_{V^*}:\dom(\part_{V^*})\to \dom(V^*)$), 
    and intervention mappings
     $\{\omega_{\widetilde{V}^*}\}_{\widetilde{V}^*\in \widetilde{\rV}^*}$ 
     ($\omega_{\widetilde{V}^*}: \dom(\part_{\widetilde{V}^*})\rightharpoonup \dom(\widetilde{V}^*)$),
    if 
    \begin{align}\label{eq: consistency}
        \mathbb{P}_{\Sol(\widetilde{\mathcal{M}};\widetilde{\ry})}(\tau(\rV))=
        \mathbb{P}_{\Sol(\widetilde{\mathcal{M}}^*;\omega(\widetilde{\ry}))}(\rV^*),
    \end{align}
    for any subset $\widetilde{\rY}^*\subseteq \widetilde{\rV}^*$, 
    and intervention $\widetilde{\ry}\in \text{dom}(\omega_{\widetilde{\rY}^*})$. 
    \end{definition}

Notice that \definitionref{def: collective agency} 
in particular requires observational consistency,
 that is, $\mathbb{P}_{\Sol(\widetilde{\mathcal{M}})}(\tau(\rV))
 =\mathbb{P}_{\Sol(\widetilde{\mathcal{M}}^*)}(\rV^*)$, 
 by considering the empty subset $\widetilde{\rY}^*=\emptyset$.

 Analogously to \citet{beckers2019abstracting}, we define strong mechanized abstractions by adding the requirement that 
 for any setting of a subset of high-level mechanisms,
 there exists a low-level intervention that implements it.

\begin{definition}
     We say that $\mec^*$ is a \textbf{strong mechanized abstraction} of $\mec$ under $\{\tau_{V^*}\}_{V^*\in \rV^*}$ and $\{\omega_{\widetilde{V}^*}\}_{\widetilde{V}^*\in \widetilde{\rV}^*}$ if
    $\mec^*$ is a mechanized abstraction of $\mec$ under $\{\tau_{V^*}\}_{V^*\in \rV^*}$ and $\{\omega_{\widetilde{V}^*}\}_{\widetilde{V}^*\in \widetilde{\rV}^*}$, and each $\omega_{\widetilde{V}^*}$ 
    is surjective onto $\dom(\widetilde{V}^*)$.
\end{definition}

If the mechanism models $\widetilde{\mathcal{M}}$ and $\widetilde{\mathcal{M}}^*$  
have no edges, then the definition above is an extension of constructive abstractions
\citep{beckers2019abstracting} that includes soft interventions a (hard intervention on 
mechanism variables may correspond to a soft intervention on the object variables). 
 Unlike previous work on abstractions of soft interventions \citep{massidda2023causal},
we do not impose any restrictions on $\omega$ except that it is induced from intervention mappings on the variable alignment. We believe that any further restrictions on $\omega$ must 
derive from context-specific considerations rather than being part of the definition. 
There are several ways in which this definition can be extended to approximate abstractions \citep{beckers2020approximate,rischel2021compositional} 
and we likewise believe that the appropriate generalization is context-dependent. In \Cref{sec:surrogate_modeling}, 
we consider restrictions on $\omega$ and a notion of approximate abstraction based on mean squared error for a fixed distribution over low-level
interventions. In other applications,
it is likely appropriate to consider KL divergence as is done, for example, in \citet{dyer2024a}.     

\subsection{Ruling out emergence of non-trivial agency}

We may want to rule out the emergence of agency at higher levels of abstraction. In this section, we provide a 
first example of a theorem that rules out that a node is a high-level agent based on properties of the 
low level. 

We say that a mechanism node $\widetilde{S}$ has an \textbf{independent mechanism} 
 if there exists a setting
$\widetilde{s}\in \dom(\widetilde{S})$ such that 
$\mathcal{F}_{\widetilde{S}}(\widetilde{\rc})=\widetilde{s}$ for all 
$\widetilde{\rc}\in \dom(\widetilde{\rV}\backslash\{\widetilde{S}\})$.

\begin{proposition}\label{prop:no_emergence_nontrivial_agency}
\Copy{prop:no_emergence_nontrivial_agency}{Let mechanized models $\mec$ and $\mec^*$ be given. Assume that $\mec^*$ is a strong
    mechanized abstraction of $\mec$ for some value mappings $\{\tau_{V^*}\}_{V^*\in \rV^*}$ and 
    intervention mappings $\{\omega_{\widetilde{V}^*}\}_{\widetilde{V}^*\in \widetilde{\rV}^*}$. 
    Let $\widetilde{S}^*\in \widetilde{\rV}^*$ be a high-level mechanism variable. Assume 
    that (i) $\tau_{\PA_{S^*}}$ is injective, and (ii) the mechanism nodes in $\part_{\widetilde{S}^*}$ have independent mechanisms.
    Then $\widetilde{S}^*$ is not a non-trivial agent in $\mec^*$. 
}
[See \Cref{app:proof_no_emergence_nontrivial_agency} for the proof.]
\end{proposition}

We think that having a more complete theory of the (non-)emergence of agency is an important direction for future work.

\subsection{Example: Actor-Critic Agents (Revisited)}
\label{sec: actor-critic}
\begin{figure}
    \centering
    \resizebox{0.85\linewidth}{!}{
    \begin{tikzpicture}[
  x=3.3cm, y=2.2cm, >=Latex,
  every node/.style={font=\Large},   %
  var/.style   ={circle,draw,fill=white,minimum size=12mm,inner sep=0pt,thick},
  tilde/.style ={rounded corners=2pt,draw,fill=black,text=white,minimum width=12mm,minimum height=12mm,inner sep=0pt,thick},
  edge/.style      ={->,thick},
  grayedge/.style  ={->,thick,draw=gray!60},
  dashededge/.style={->,thick,dashed}
]
\tikzset{pics/model/.style={code={
  \node[] at (-0.5,2.2) {$\mec$};
  \node[tilde] (tA) at (0,2)  {$\widetilde{A}$};
  \node[tilde] (tS) at (1,2)  {$\widetilde{S}$};
  \node[tilde] (tR) at (2,2)  {$\widetilde{R}$};

  \node[var]   (A)  at (0,1)  {$A$};
  \node[var]   (S)  at (1,1)  {$S$};
  \node[var]   (R)  at (2,1)  {$R$};

  \node[var]   (Q)  at (0,0)  {$Q$};
  \node[var]   (Y)  at (1,0)  {$Y$};
  \node[var]   (W)  at (2,0)  {$W$};

  \node[tilde] (tQ) at (0,-1) {$\widetilde{Q}$};
  \node[tilde] (tY) at (1,-1) {$\widetilde{Y}$};
  \node[tilde] (tW) at (2,-1) {$\widetilde{W}$};

  \draw[edge] (A) -- (Y);
  \draw[edge] (A) -- (S);
  \draw[edge] (Q) -- (Y);
  \draw[edge] (S) -- (R);
  \draw[edge] (R) -- (W);
  \draw[edge] (Y) -- (W);

  \draw[grayedge] (tY) -- (Y);
  \draw[grayedge] (tR) -- (R);
  \draw[grayedge] (tQ) -- (Q);
  \draw[grayedge] (tA) -- (A);
  \draw[grayedge] (tS) -- (S);
  \draw[grayedge] (tW) -- (W);

  \draw[dashededge] (tS) -- (tQ);
  \draw[dashededge] (tR) to[bend left=30] (tQ.east);
  \draw[dashededge] (tQ) to[bend left=20] (tA.south west);

  \coordinate (-center) at (1,0.5); %
  \coordinate (-west)   at (-.4,0.5);
  \coordinate (-east)   at ( 2.4,0.5);
}}}

\tikzset{pics/model1/.style={code={
\node[] at (-0.5,1.5) {$\mec^*$};
  \node[tilde] (tA) at (0,1)  {$\widetilde{A}^*$};
  \node[tilde] (tS) at (1,1)  {$\widetilde{S}^*$};
  \node[tilde] (tR) at (2,1)  {$\widetilde{R}^*$};

  \node[var]   (A)  at (0,0)  {$A^*$};
  \node[var]   (S)  at (1,0)  {$S^*$};
  \node[var]   (R)  at (2,0)  {$R^*$};

  \draw[edge] (S) -- (R);
  \draw[edge] (A) -- (S);

  \draw[grayedge] (tR) -- (R);
  \draw[grayedge] (tA) -- (A);
  \draw[grayedge] (tS) -- (S);

  \draw[dashededge] (tS) -- (tA);
  \draw[dashededge] (tR) to[bend right=30] (tA.north east);

  \coordinate (-center) at (1,0.5); %
  \coordinate (-west)   at (-.4,0.5);
  \coordinate (-east)   at ( 2.4,0.5);
}}}

\pic (M1) at (0,0) {model};
\pic (M2) at (3.8,0) {model1}; %

\draw[very thick,->] (M1-east) -- node[above, scale=1]{$\tau,\omega$} (M2-west);

\end{tikzpicture}
}
    \caption{A graphical representation of a causal abstraction applied to mechanized SCMs representing the AC example from \Cref{ex: actor-critic introduction}  (see \Cref{app: actor-critic details} for the explicit construction of the abstraction).
    On the left hand side, 
    we have the low-level model $\mec$.
     On the right, we have a high-level model $\mec^*$, which 
     is a valid abstraction of the low-level model 
     under $\tau$ and $\omega$. 
     By examining the high-level model, we can see
      that $A^*$ is an $\mathcal{R}^{\text{BR}}$-rational 
      agent with utility equal to the reward: $\mathcal{U}(\rv^*)=\proj_{R^*}(\rv^*)$.
    }
    \label{fig: actor-critic}
\end{figure}
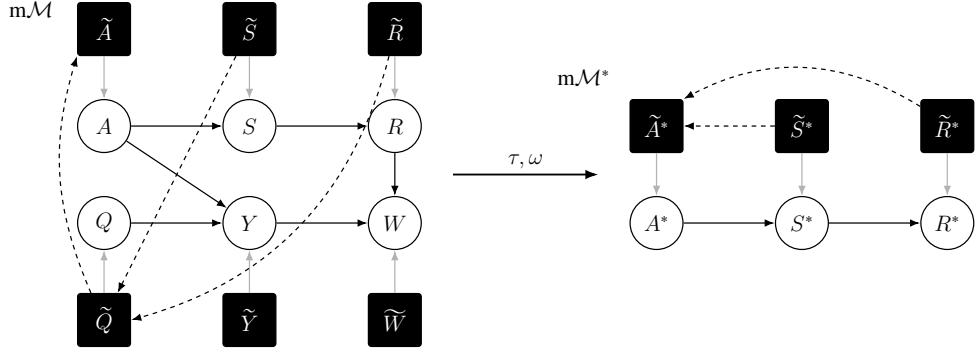
    We revisit the one-step AC agent (as described in \citet{kenton2023discovering} and \Cref{ex: actor-critic introduction}).
    Here the critic $\widetilde{Q}$ is tasked with predicting 
    the reward $R$ for different 
    actions that the actor $\widetilde{A}$ may take.
    The action $A$ taken by the
    actor $\widetilde{A}$ affects the state, 
    and in turn, the reward.
    \citet{kenton2023discovering} hypothesize that one can, 
    at a more coarse-grained level, 
    view this as a single-agent system.
    This in turn explains why 
    the system intuitively has an incentive 
    to manipulate the state $S$, 
    even though if one applies the graphical criterion for 
    single-agent graphs \citep{everitt2021agent}, it suggests 
    that neither $\widetilde{A}$ nor $\widetilde{Q}$  
    have an incentive to manipulate the state $S$.  See \Cref{app: actor-critic details} for details about
    how to formalize this setup in terms of mechanized SCMs.

\begin{proposition}\label{prop:actor_critic_abstraction}
\Copy{prop:actor_critic_abstraction}{In \Cref{fig: actor-critic} 
(see also \Cref{app: actor-critic details} for details), $\mec^*$
 is a strong mechanized abstraction of $\mec$. 
 Furthermore, $\widetilde{A}^*$ in the high-level 
 model $\mec^*$ is an $(\mathcal{R}^{\text{BR}},\mathcal{U})$-agent
  with utility equal to the reward $\mathcal{U}(\rv^*)=\proj_{R^*}(\rv^*)$.} 
[See \Cref{app:proof_actor_critic} for the proof.] 
\end{proposition}

Note that while the AC example  `marginalizes' out one of the agents (namely 
the critic), our framework allows for more general types of collective agency. 
For example, in the next section, we consider aggregating individuals into a single agent. 

\section{Example: Using Collective Agency for Surrogate Modeling}\label{sec:surrogate_modeling}

\begin{figure}[ht]
    \centering
    \resizebox{0.8\textwidth}{!}{
        \begin{tikzpicture}[
  node distance=1.0cm and 1.0cm,
  >={Latex}, %
  semithick,
  every node/.style={font=\small},
  var/.style   ={circle,draw,fill=white,minimum size=8mm,inner sep=0pt,semithick},
  mech/.style ={rounded corners=2pt,draw,fill=black,text=white,minimum width=8mm,minimum height=8mm,inner sep=0pt,semithick},
  edge/.style      ={->,semithick},
  grayedge/.style  ={->,semithick,draw=gray!60},
  dashededge/.style={->,semithick,dashed},
  container/.style={draw=gray!50, dashed, rounded corners, inner sep=0.3cm}
]

\node[mech] (mt11) {$\widetilde{v}_{1,1}$};
\node[var, below=0.4cm of mt11] (v11) {$v_{1,1}$};
\draw[grayedge] (mt11) -- (v11);

\node[mech, right=0.5cm of mt11] (mt12) {$\widetilde{v}_{1,2}$};
\node[var, below=0.4cm of mt12] (v12) {$v_{1,2}$};
\draw[grayedge] (mt12) -- (v12);

\node[right=0.2cm of mt12] (d1t) {$\dots$};
\node[right=0.2cm of v12] (d1v) {$\dots$};

\node[mech, right=0.2cm of d1t] (mt1n) {$\widetilde{v}_{1,n}$};
\node[var, below=0.4cm of mt1n] (v1n) {$v_{1,n}$};
\draw[grayedge] (mt1n) -- (v1n);

\node[var, below=1.5cm of v12] (q1) {$q_1$};
\node[mech, left=0.8cm of q1] (mq1) {$\widetilde{q}_1$};
\draw[grayedge] (mq1) -- (q1);

\draw[edge] (v11) -- (q1);
\draw[edge] (v12) -- (q1);
\draw[edge] (v1n) -- (q1);

\node[mech, right=3.5cm of mt1n] (mt21) {$\widetilde{v}_{2,1}$};
\node[var, below=0.4cm of mt21] (v21) {$v_{2,1}$};
\draw[grayedge] (mt21) -- (v21);

\node[mech, right=0.5cm of mt21] (mt22) {$\widetilde{v}_{2,2}$};
\node[var, below=0.4cm of mt22] (v22) {$v_{2,2}$};
\draw[grayedge] (mt22) -- (v22);

\node[right=0.2cm of mt22] (d2t) {$\dots$};
\node[right=0.2cm of v22] (d2v) {$\dots$};

\node[mech, right=0.2cm of d2t] (mt2n) {$\widetilde{v}_{2,n}$};
\node[var, below=0.4cm of mt2n] (v2n) {$v_{2,n}$};
\draw[grayedge] (mt2n) -- (v2n);

\node[var, below=1.5cm of v22] (q2) {$q_2$};
\node[mech, right=0.8cm of q2] (mq2) {$\widetilde{q}_2$};
\draw[grayedge] (mq2) -- (q2);

\draw[edge] (v21) -- (q2);
\draw[edge] (v22) -- (q2);
\draw[edge] (v2n) -- (q2);

\node[var] (u_all) at ($(q1)!0.5!(q2)$) {$\bm{U}$};
\node[mech, above=0.4cm of u_all] (mu_all) {$\widetilde{\bm{U}}$};
\draw[grayedge] (mu_all) -- (u_all);

\node[right=0.0cm of mu_all] (hammer) {
    \tikz[baseline, scale=0.3, rotate=25]{
        \draw[fill=green!40, semithick] (-0.2, 0) rectangle (0.2, 1.5); %
        \draw[fill=green!60, semithick] (-0.8, 1.5) rectangle (0.8, 2.2); %
    }
};

\draw[edge] (q1) -- (u_all);
\draw[edge] (q2) -- (u_all);

\node[container, fit=(mt11) (mt1n) (q1) (mq1), label={[font=\small]south:Low-Level Country 1}] (c1_box) {};
\node[container, fit=(mt21) (mt2n) (q2) (mq2), label={[font=\small]south:Low-Level Country 2}] (c2_box) {};

\coordinate (bar_left) at ([yshift=1.0cm]mt11.north);
\coordinate (bar_right) at ([yshift=1.0cm]mt2n.north);

\draw[dashed, semithick] (bar_left) -- (bar_right) node[midway, above, font=\footnotesize] {Global Interaction};

\draw[dashededge] (mt11 |- bar_left) -- (mt11);
\draw[dashededge] (mt12 |- bar_left) -- (mt12);
\draw[dashededge] (mt1n |- bar_left) -- (mt1n);

\draw[dashededge] (mt21 |- bar_left) -- (mt21);
\draw[dashededge] (mt22 |- bar_left) -- (mt22);
\draw[dashededge] (mt2n |- bar_left) -- (mt2n);

\draw[dashed, semithick] 
    (mq1.north) to[out=110, in=180] (bar_left);

\draw[dashed, semithick] 
    (mq2.north) to[out=45, in=0, looseness=1.3] (bar_right);

\draw[dashed, semithick] 
    (mu_all.north) to[out=90, in=270] ($(bar_left)!0.5!(bar_right)$);

\node[mech, below=2.5cm of q1] (M1) {$\widetilde{q}_1$};
\node[var, below=0.4cm of M1] (Q1) {$q_1$};
\draw[grayedge] (M1) -- (Q1);

\node[mech, below=2.5cm of q2] (M2) {$\widetilde{q}_2$};
\node[var, below=0.4cm of M2] (Q2) {$q_2$};
\draw[grayedge] (M2) -- (Q2);

\node[container, fit=(M1) (Q1), label={[font=\small]south:High-Level Country 1}] (C1_high) {};
\node[container, fit=(M2) (Q2), label={[font=\small]south:High-Level Country 2}] (C2_high) {};

\node[var] (U_high) at ($(Q1)!0.5!(Q2)$) {$\bm{U}$};
\node[mech, above=0.4cm of U_high] (MU_high) {$\widetilde{\bm{U}}$};
\draw[grayedge] (MU_high) -- (U_high);

\node[right=0.0cm of MU_high] (hammer_high) {
    $\omega\left(
    \tikz[baseline={([yshift=-0.6ex]current bounding box.center)}, scale=0.3, rotate=25]{
        \draw[fill=green!40, semithick] (-0.2, 0) rectangle (0.2, 1.5); %
        \draw[fill=green!60, semithick] (-0.8, 1.5) rectangle (0.8, 2.2); %
    }
    \right)$
};

\draw[edge] (Q1) -- (U_high);
\draw[edge] (Q2) -- (U_high);

\draw[dashededge, <->, bend left=20] (M1) to node[midway] (interaction_center) {} (M2);
\draw[dashed, semithick] (MU_high) -- (interaction_center);

\coordinate (center_low) at ($(q1)!0.5!(q2)$);
\coordinate (center_high) at ($(Q1)!0.5!(Q2)$);

\coordinate (arrow_start) at (center_low |- c1_box.south);
\coordinate (arrow_end) at (center_high |- C1_high.north);

\draw[->, thick] ([yshift=-0.2cm]arrow_start) -- node[midway, right] {$\tau, \omega$} ([yshift=0.4cm]arrow_end);

\path (bar_left) ++(0,0.5) coordinate (top_margin);

\end{tikzpicture}
    }
    \caption{At the low level,
    the global interaction includes edges between every pair of citizens and from the voting mechanism
    $\widetilde{q}_1$ and $\widetilde{q}_2$ to every citizen. At the high level, we view the countries as
    being individual agents, picking the pollution level in response to the other countries.
    Since we are intervening on the citizens preferences, we explicitly include a (vector valued) utility
    node $\widetilde{\bm{U}}$.
    At the low level, it has dimension equal to the number of citizens. At the
    high level, it has dimension equal to the number of countries. In the figure and the main text, we omit superscripts
    (*) for the high level.}
    \label{fig:voting_abstraction}
\end{figure}
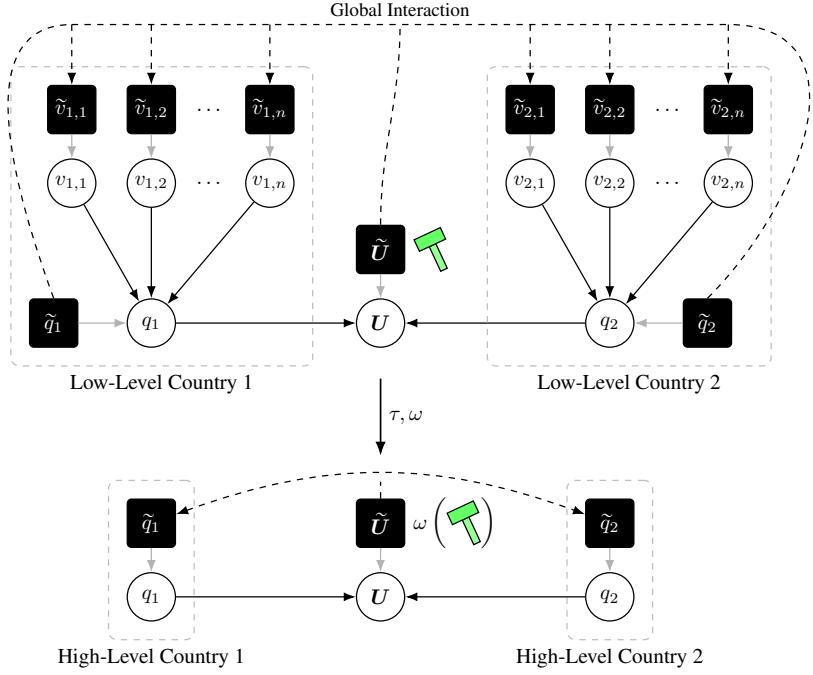

In this section we show an application of collective agency for 
surrogate modeling. Specifically, we investigate the degree to which it makes
sense to model countries as collective agents under various voting mechanisms. 
\citet{dyer2024a} introduced the idea of using
causal abstractions for interventionally consistent surrogate modeling. 
We build on this work by extending it to mechanized abstractions with agents. 
There are at least three potential benefits of using surrogate models with collective agents:
1) the low-level model is often a black box where new interventions cannot be simulated;\footnote{This is essentially the point of discovering psychological regularities. It is much easier to predict people's behavior based
on abstract models of beliefs and preferences than based on neurophysiology \citep{dennett1989intentional}.}
2) computing low-level equilibria may be computationally expensive, which hinders evaluations of new interventions; and
3) agency-based surrogate models may be useful for assessing abstract counterfactuals such as ``how would country A react if country B adopts policy X?''.

While we believe that there are practical applications where surrogate modeling with agents is useful, the main point of the example below is to provide a theoretical illustration and proof of concept. We therefore prioritize simplicity over realistic assumptions.

\subsection{The Low-Level Model}
We have $C$ countries and each country $c$ has $N_c$ citizens. 
Each country employs a voting mechanism to determine the level of pollution $q_c$. We assume that individual $i$ from country $c$ has utility
$$U_{ci}(q_c,q_{-c})=a_{ci}q_c-b_{ci}q_c^2-d_{ci}Q_W^2,$$
where $q_{-c}$ are the pollution levels of countries other than $c$, $a_{ci},b_{ci},d_{ci}\in \mathbb{R}$ are individual parameters, and $Q_W=\sum_c q_c$.
We consider interventions on individuals' preferences for pollution, 
for example, implemented by taxation or advocacy, 
such that  $a_{ci}$ is replaced with $a_{ci}-\lambda_{ci}\geq 0$.  

We assume that the result of 
the countries' votes is a 
Nash equilibrium (NE) where every citizen of every country is a player. 
This NE may be computationally expensive to evaluate or, if the low-level model is not understood, can only be evaluated by playing it out in the physical world.
Rather than reasoning about NEs containing every 
citizen of every country, it may be possible to 
model the situations at a higher level of 
abstraction where countries are modeled as agents.

\subsection{The High-Level Model}
We use a parametric model for country
 utilities, which mimics the utilities of individual agents:
\begin{align}\label{eq: high level utility} 
U_c(q_c,q_{-c})=A_cq_c-B_cq_c^2-D_cQ_W^2.
\end{align}

Assuming that each country is an agent with this utility function,
it is easy to derive that under NE, the total amount of pollution is
$$Q_W^{\text{NE}}=\frac{\frac{1}{2}\sum_c\frac{A_c}{B_c}}{1+\sum_c\frac{D_c}{B_c}},$$ and the individual contributions are
\begin{align}\label{eq: ne}
q_c^{\text{NE}}
= \frac{A_c}{2B_c}
-\frac{D_c}{B_c}Q_W^{\text{NE}}.
\end{align}

Notice that there is no guarantee that the voting mechanisms will 
produce pollution levels that are correctly modeled 
by the countries playing an NE according
 to a utility function parameterized by \Cref{eq: high level utility}.
This is a modeling choice that may or 
may not be appropriate depending on the actual voting
 mechanism that is used and the utility functions of the individual citizens. 

\subsection{Connecting the Low-Level and High-Level Models}\label{connection}
We want to learn a mapping $\omega$ from the vector intervention $\bm{\lambda}:=(\lambda_{ci})_{ci}$ to an intervention on $(A_c,B_c,D_c)_c$ such that the high-level model is a mechanized abstraction of the low-level model (see \Cref{fig:voting_abstraction}).
In other words, we want to see if we can view the countries as collective agents.
Formally, $\omega$ ought to be such that
$$\mathbb{P}_{\Sol(\widetilde{\mathcal{M}}; \bm{\lambda})}((q_c)_c)
= \mathbb{P}_{\Sol(\widetilde{\mathcal{M}}^*; \omega(\bm{\lambda}))}((q_c)_c)$$
where we assume that solutions are given by best response rationality.
In words, we want that applying intervention $\bm{\lambda}$ to the citizen preferences results in a low-level NE that equals the high-level NE under the $\omega(\bm{\lambda})$ intervention in the high-level model.  

Notice that the high-level NE only depends on the parameters $A_c$,$ B_c$, and $D_c$ through the fractions $\alpha_c:=\frac{A_c}{B_c}$ and $\delta_c:=\frac{D_c}{B_c}$. 
A single sample
is not enough to identify $\alpha_c$ and $\delta_c$
 (for each country, we have a single equation -- \Cref{eq: ne} -- and two unknowns).
 Therefore we need to make assumptions about how our intervention 
 affects the high-level utility function. For simplicity, we assume that the interventions only 
 affect the $\alpha_c$'s and not the $\delta_c$'s (which we know to be true for VCG voting; see \Cref{app: details voting mechanisms}). 
 Whether this assumption is reasonable 
 for other voting mechanisms is an empirical question, which we explore in the results
 section.

Using this assumption, we estimate $\omega$ using a two-step procedure.
First, for each country $c$, 
we apply interventions such that $\lambda_{ci}=0$ for all citizens $i$ of country $c$,
but non-zero for other countries.
We then estimate $\delta_c$ by linear 
regression of $q_c$ on $Q_W$ (which is justified by \Cref{eq: ne}). 
Second, we parametrize $\omega:\mathbb{R}^{\sum_c N_c}\to \mathbb{R}^C$ 
by the weights $\phi$ of a neural network and train the network to minimize the loss function
\begin{align*}
    \text{Loss}(\phi) = \sum_j \sum_c \left[ \widehat{q}_c^{\text{NE}}(\phi(\bm{\lambda}^j)) - q_c ^{\text{NE}}(\bm{\lambda}^j) \right]^2,
\end{align*}
where $\bm{\lambda}^j$ is an intervention on the citizens preferences, $q_c^{\text{NE}}(\bm{\lambda}^j)$ 
is the resulting ground-truth low-level NE,
and $\widehat{q}_c^{\text{NE}}(\phi(\bm{\lambda}^j))$ is the
NE derived by \Cref{eq: ne} and
the estimated $\widehat{\delta_c}$ and $\widehat{\alpha_c}$,
 the latter of which is one of the outputs returned by $\omega$. 
 The interventions $\bm{\lambda}^j$ are sampled according to a fixed distribution (see \Cref{app: details of voting experiment} for details).
 Minimizing the consistency loss is the central 
idea of approximate causal abstractions \citep{beckers2020approximate,rischel2021compositional}.

\subsection{Results}

We consider three different voting mechanisms: Vickrey-Clarke-Groves (VCG), 
Median Voting, and Random Dictator. 
Descriptions of the three voting mechanisms and how we compute low-level Nash equilibria for each
are given in \Cref{app: details voting mechanisms}.
We apply our method to a simulation with five countries and a total of 1,000 citizens.
See \Cref{app: details of voting experiment} for details about the simulations and training. 
We report a summary of the results in \Cref{table: mechanism comparison} and 
provide detailed results in \Cref{app: voting experiment results}.

\begin{table}[h]
    \centering
    \begin{tabular}{l|c|c|c}
        Mechanism & Model MAE & Baseline MAE & Improvement \\ \hline
        VCG & 0.062 & 1.254 & 95.0\% \\
        Median & 0.142 & 1.248 & 88.6\% \\
        Random Dictator & 5.077 & 4.645 & -9.3\% \\
    \end{tabular}
    \caption{We are predicting $q_c$ for different voting mechanisms under 
    interventions $\bm{\lambda}$. 
    (MAE) is mean absolute error for different voting mechanisms. The baseline is calculated by constantly predicting the low-level Nash
     equilibrium at $\lambda=0$ (since random dictator is stochastic we use the average 
     low-level NE for this mechanism as the baseline). The errors are summed across all countries.
      VCG and Median voting show low errors (compared to the baseline),
     indicating that the collective agency framework effectively captures how interventions affect outcomes.
      Random Dictator performs worse than the baseline (-9.3\%) and has large absolute errors,
      demonstrating the difficulty of learning an $\omega$ for this mechanism. 
      }
    \label{table: mechanism comparison}
\end{table}

\section{Summary and future work}
\label{sec:conclusion}

In this work we have laid out a mathematical foundation 
of agency and collective agency based on the mathematical formalisms of
 mechanized causal graphs and causal abstraction. 
The framework opens up many interesting questions that we hope can be addressed in future work, including:
1) Investigating how different notions of rationality, such as different decision theories, 
affect the propensity for collective agency;
2) Developing a theory that predicts
 when collective agency will emerge (see \propositionref{prop:no_emergence_nontrivial_agency});
3) Extending inverse game theory or inverse reinforcement
 learning to various levels of abstraction; and
4) Translating our theory of collective agency into a mathematical framework that more naturally accommodates both non-independent mechanisms and causal abstractions, such as \citet{garrabrant2024factored}.
We also hope that our work can help to provide the foundation for further theoretical and empirical work to understand, predict, 
    and control emergent collective
    agents in networks of 
    AI systems.

\section{Acknowledgments}
This work was initiated during FHJ's participation in the Pivotal Research Fellowship.
We thank Casper Lützhøft Christensen, Soroush Ebadian, 
Sukanya Krishna, Francisco Madaleno, Kyle Reynoso, Morgan Simpson, and Riya Tyagi for 
helpful discussions and feedback.

\bibliography{bib_file}

\appendix

\section{Additional Related Work}
\label{app:related_work}

Due to space constraints, our discussion of related work in \Cref{sec:related_work} was rather brief. Here we extend this discussion along several axes.

\paragraph{Philosophical Foundations of Collective Agency.}
The question of when multiple individuals constitute a unified collective agent has long occupied philosophers.
Early foundational work established the importance of collective intentionality -- the capacity for shared intentions and joint commitments -- in grounding group agency \citep{Searle1990,Tuomela2006,Bratman2014}.
Subsequent work has refined these accounts, examining the semantics of collective intentional behavior \citep{Ludwig2007}, distinguishing between different modes of shared intention \citep{Pacherie2013}, and exploring the relationship between collective intentions and team agency \citep{Gold2007}.
\citet{List2011} provide a more functionalist account, arguing that a group can be considered an agent in its own right when it has representational states, motivational states, and a capacity to process these states, even when individual members may lack shared beliefs or desires.
A key challenge across these accounts is the \emph{aggregation problem}: determining when coordinated behavior among individuals reflects genuine collective agency rather than merely the sum of individual actions \citep{Anderson1972,Roth2017}.
This question becomes particularly acute when considering systems that exhibit goal-directed behavior at multiple levels of abstraction, from neural subsystems within individual minds \citep{Minsky1988} to emergent patterns in complex adaptive systems \citep{Mitchell2009}.
It is worth noting that the question of collective \emph{agency} is distinct from, though related to, that of collective \emph{intelligence} \citep{Malone2015}, the latter focusing more on problem-solving capabilities than on unified agency per se.
While these philosophical frameworks provide important conceptual foundations, they generally lack the formal machinery needed to rigorously identify collective agents or predict when they will emerge.

\paragraph{Causal Definitions of Individual Agency.}
Recent work has made important progress on rigorously defining and identifying \emph{individual} agents using tools from causality and decision theory.
Essentially, this is because any notion of goal-directedness -- which is intrinsic to agency -- requires considering what a putative agent would have done differently to achieve their goal, had their circumstances been different.
Early work in AI and multi-agent systems developed formal models of joint action and cooperation \citep{Levesque1990}, though these typically assumed rather than derived the conditions for collective agency.
Most closely related to our work, \citet{kenton2023discovering} propose a causal discovery algorithm for identifying agents from empirical data.
Others define and detect agency based on how predictive the assumption of utility-maximizing behavior is of a system's observed behavior: \citet{Orseau2018} use Bayesian inverse reinforcement learning \citep{Ramachandran2007} to infer the extent to which a given system is a goal-directed agent, whereas \citet{macdermott2024measuring} use a maximal causal entropy model \citep{Ziebart2010} to measure goal-directedness.
More recently, others have attempted to apply such definitions to measure the goal-directedness not just of RL but also LLM-based agents \citep{Xu2024,Everitt2025}.
A related but separate research thread has studied causal definitions of intention \citep{Halpern2018,Ward2024}.
None of these works, however, formally consider the question of \emph{collective} agents, where multiple subsystems have different actions, observations, and incentives.

\paragraph{Multi-Agent Formalisms and Collective Behavior.}
For settings with multiple agents, \citet{hammond2023reasoning} provide a formal framework -- known as causal games -- by generalizing multi-agent influence diagrams \citep{Koller2003} to the higher levels of Pearl's causal hierarchy \citep{Pearl2009a}.%
\footnote{In the same paper, \citet{hammond2023reasoning} also introduce \textit{mechanized} causal games, in which the parameters or decision rules of variables are explicitly represented as mechanism nodes, allowing \citet{kenton2023discovering} to discover which elements correspond to decisions and utilities.}
Other game-theoretic approaches, such as coalition structures in cooperative game theory \citep{Ray2007,Elkind2016}, provide models of when subgroups might form and what outcomes they might achieve, but do not address whether such coalitions constitute unified agents.
Complementing these formal frameworks, a substantial literature examines collective intelligence and emergent behavior in multi-agent systems, from swarm algorithms \citep{Bonabeau1999} to collective decision-making in biological groups \citep{Couzin2009}.
Several works have developed formal methods for detecting or measuring emergent phenomena by comparing micro- and macro-level descriptions \citep{Kubik2003,Seth2006,Szabo2015} or by measuring changes in collective capabilities \citep{Sourbut2024}.
Others have proposed theoretical frameworks for emergent agency drawing on active inference \citep{Friston2022} or evolutionary perspectives \citep{Okasha2018,Smith2020}.
However, these approaches either measure emergence empirically without providing normative criteria for when collective agency exists, or apply to specific domains without offering a general framework for identifying collective agents across contexts.

\paragraph{Causal Abstraction.}
The notion that systems can be accurately described at multiple levels of abstraction has been formalized in the causal modelling literature through the concept of causal abstraction.
\citet{rubenstein2017causal} 
and
\citet{beckers2019abstracting} 
develop frameworks for determining when a high-level causal model is a valid abstraction of a more detailed low-level model, preserving causal relationships even when variables are aggregated or details are omitted.
\citet{beckers2020approximate} extend this to allow for approximate abstractions that sacrifice some precision for greater simplicity.
These frameworks provide rigorous criteria for when a simplified model faithfully represents the causal structure of a more complex system -- precisely the kind of tool needed to determine when viewing a group of agents as a single collective agent is justified.
However, existing causal abstraction work has focused primarily on single-agent causal influence diagrams or on causal models without decision-makers.
Applying these ideas to multi-agent settings, where the low-level model involves multiple decision-makers with potentially conflicting objectives, requires extending the abstraction framework to handle the strategic interactions captured by causal games.

\paragraph{AI Safety and Multi-Agent Risks.}
Our primary motivation is the safety and control of networks of advanced AI agents, and in particular the possibility of emergent `super-agents' with dangerous or unexpected capabilities or goals.
A key safety challenge is that multiple, simpler AI systems might inadvertently form a collective agent with capabilities and objectives distinct from any individual in the group \citep{hammond2025multi}.
Compared to AI tools, the ability of artificial \emph{agents} to make plans and take actions in pursuit of complex goals makes them not only more useful, but also more harmful if those goals are misaligned \citep{chan2023harms}.
For example, competitive pressures may lead individually rational AI agents to rapidly exhaust collective resources \citep{Piatti2024}, or a group of agents might combine their harmless individual capabilities to override safeguards and execute dangerous attacks \citep{Jones2024}.
Several works have examined emergent capabilities in multi-agent AI systems \citep{Chen2009,Teo2013}, though most focus on identifying specific emergent behaviors rather than general conditions for collective agency.
\citet{drexler2019reframing} proposes that a comprehensive set of narrow AI `services' could provide a safer alternative to monolithic superintelligence, though this depends on preventing unwanted collective agents from emerging.
Understanding when and how collective agency arises is therefore critical for anticipating and mitigating multi-agent AI risks.

\section{Battle of the sexes example}\label{app: battle of the sexes}

\begin{example}\textbf{Battle of the sexes.} We illustrate the formalism with the classic 
    two-player coordination game `battle of the sexes', shown in \Cref{table: battle}. 

\begin{table}[h]
    
  \centering
  \caption{Battle of the sexes payoff matrix}
  \begin{tabular}{c|cc}
           & Opera & Football \\ \hline
    Opera  & $(2,1)$ & $(0,0)$ \\
    Football & $(0,0)$ & $(1,2)$ \\
  \end{tabular}
  \label{table: battle}
\end{table}

\begin{figure}[h]
    \centering
    \resizebox{0.4\linewidth}{!}{
    \begin{tikzpicture}[
      x=2.5cm, y=2.0cm, >=Latex,
      every node/.style={font=\Large},
      var/.style   ={circle,draw,fill=white,minimum size=12mm,inner sep=0pt,thick},
      tilde/.style ={rounded corners=2pt,draw,fill=black,text=white,minimum width=12mm,minimum height=12mm,inner sep=0pt,thick},
      edge/.style      ={->,thick},
      grayedge/.style  ={->,thick,draw=gray!60},
      dashededge/.style={->,thick,dashed}
    ]

    \node[tilde] (tD1) at (0,2) {$\widetilde{D}_1$};
    \node[tilde] (tD2) at (2,2) {$\widetilde{D}_2$};

    \node[var] (D1) at (0,1) {$D_1$};
    \node[var] (D2) at (2,1) {$D_2$};
    \node[var] (U1) at (0,0) {$U_1$};
    \node[var] (U2) at (2,0) {$U_2$};

    \node[tilde] (tU1) at (0,-1) {$\widetilde{U}_1$};
    \node[tilde] (tU2) at (2,-1) {$\widetilde{U}_2$};

    \draw[grayedge] (tD1) -- (D1);
    \draw[grayedge] (tD2) -- (D2);
    \draw[grayedge] (tU1) -- (U1);
    \draw[grayedge] (tU2) -- (U2);

    \draw[edge] (D1) -- (U1);
    \draw[edge] (D1) -- (U2);
    \draw[edge] (D2) -- (U1);
    \draw[edge] (D2) -- (U2);

    \draw[dashededge] (tU1) to[bend left=45] (tD1); 
    \draw[dashededge] (tU2) to[bend right=45] (tD2); 
    \draw[dashededge] (tD1) to[bend left=15] (tD2);
    \draw[dashededge] (tD2) to[bend left=15] (tD1);

    \end{tikzpicture}
    }
    \caption{Mechanized causal graph for 2x2 games (of which battle of the sexes is an 
    example). Dashed edges 
    represent edges in the mechanisms model. 
    Solid edges represent edges in the object-level model. 
    Gray edges link mechanism variables with object variables.
    The distribution of $D_1$, which 
    corresponds to the decision of player 1, depends on the conditional distribution of $U_1$ 
    given $D_1$ and $D_2$ and the marginal distribution of $D_2$. 
    This is represented by the dashed edges from $\widetilde{U}_1$ to $\widetilde{D}_1$ and 
    $\widetilde{D}_2$ to $\widetilde{D}_1$, respectively. If player 1 is choosing a 
    strategy using best response rationality (see \definitionref{def: best response rationality})
    with utility equal to the payoff, then player 1 will not respond to changes
    in $\widetilde{U}_2$ if $\widetilde{D}_2$ is kept fixed. Therefore, $\widetilde{D}_1$ does not
    functionally depend on $\widetilde{U}_2$ and we draw no dashed edge from $\widetilde{U}_2$ to $\widetilde{D}_1$.}

    \label{fig: battle}
\end{figure}
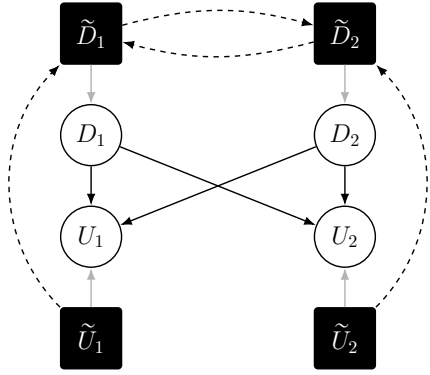

The mechanized SCM $\mec$ consists of a deterministic cyclic SCM with signature
\begin{align*}
    \widetilde{D}_1&\in [0,1]\\
    \widetilde{D}_2&\in [0,1] \ \\
    \widetilde{U}_1&\in \mathbb{Z}^{\{O,F\}\times\{O,F\}}\\
    \widetilde{U}_2&\in \mathbb{Z}^{\{O,F\}\times\{O,F\}}.
\end{align*}
$\widetilde{D}_1$ and $\widetilde{D}_2$ specify probability distributions over `opera' and 
`football' (the number in $[0,1]$ represents the probability of picking opera). 
$\widetilde{U}_1$ and $\widetilde{U}_2$ 
$: \{O,F\}\times\{O,F\}
\to\mathbb{Z}$ specify the payoffs
 of the two players' decisions (for example, in the battle of the sexes, $\widetilde{U}_1(O,O)
=2$, see \Cref{table: battle}). 
The parameterized SCM is given by the structural assignments
\begin{align*}
    \mathcal{F}_{D_1}^{\widetilde{D}_1}(\mathcal{E}_{D_1})=\begin{cases}
        O & \mathcal{E}_{D_1}\leq \widetilde{D}_1\\
        F & \text{otherwise}
    \end{cases}\\
    \mathcal{F}_{D_2}^{\widetilde{D}_2}(\mathcal{E}_{D_2})=\begin{cases}
        O & \mathcal{E}_{D_2}\leq \widetilde{D}_2\\
        F & \text{otherwise}
    \end{cases}\\
    \mathcal{F}_{U_1}^{\widetilde{U}_1}(D_1,D_2, \mathcal{E}_{U_1})=\widetilde{U}_1(D_1,D_2)\\
    \mathcal{F}_{U_2}^{\widetilde{U}_2}(D_1,D_2, \mathcal{E}_{U_2})=\widetilde{U}_2(D_1,D_2)
\end{align*}
$\mathcal{E}_{D_1},\mathcal{E}_{D_2}\sim \text{Unif}(0,1)$.
 We assume that the payoffs are
  deterministic in the decisions, 
  and so, the distributions of the 
  noise variables $\mathcal{E}_{U_1},\mathcal{E}_{U_2}$ are irrelevant. 

We now specify the mechanism model $\widetilde{\mathcal{M}}$.
We assume that the payoffs are as specified in \Cref{table: battle}, that is,
\begin{align*}
    \mathcal{F}_{\widetilde{U}_1}(\widetilde{\rV})=\begin{cases}
        (O,O)\mapsto 2\\
        (F,F)\mapsto 1\\
        (O,F)\mapsto 0\\
        (F,O)\mapsto 0
    \end{cases}\\
    \mathcal{F}_{\widetilde{U}_2}(\widetilde{\rV})=\begin{cases}
        (O,O)\mapsto 1\\
        (F,F)\mapsto 2\\
        (O,F)\mapsto 0\\
        (F,O)\mapsto 0.
    \end{cases}\\
\end{align*}
The assignments for the payoff mechanisms constantly return the payoff function specified in \Cref{table: battle} (irrespective of the policies of the players).
We assume that the $\widetilde{D}_1$ and $\widetilde{D}_2$ act as to optimize the expected value of their respective payoffs, that is, 
\begin{align*}
    \mathcal{F}_{\widetilde{D}_1}(\widetilde{D}_2,\widetilde{U}_1,\widetilde{U}_2)\in\underset{\widetilde{d}_1\in [0,1]}{\argmax} \  \mathbb{E}_{\widetilde{U}_1,\widetilde{U}_2,\widetilde{d}_1,\widetilde{D}_2}( U_1 )\\
    \mathcal{F}_{\widetilde{D}_2}(\widetilde{D}_1,\widetilde{U}_1,\widetilde{U}_2)\in\underset{\widetilde{d}_2\in [0,1]}{\argmax} \  \mathbb{E}_{\widetilde{U}_1,\widetilde{U}_2,\widetilde{D}_1,\widetilde{d}_2}(U_2)
\end{align*}    
Notice that $\mathcal{F}_{\widetilde{D}_1}$ depends trivially 
on $\widetilde{U}_2$, and $\mathcal{F}_{\widetilde{D}_2}$ depends trivially 
on $\widetilde{U}_1$. 
It is well known that this system of equations has three solutions $\Sol(\widetilde{\mathcal{M}})=\{\widetilde{s}_1,\widetilde{s}_2,\widetilde{s}_3\}$ where 
\begin{align*}
\proj_{\{\widetilde{D}_1,\widetilde{D}_2\}}(\widetilde{s}_1)&=\{1_{\widetilde{D}_1},1_{\widetilde{D}_2}\}\\
\proj_{\{\widetilde{D}_1,\widetilde{D}_2\}}(\widetilde{s}_2)&=\{0_{\widetilde{D}_1},0_{\widetilde{D}_2}\}\\
\proj_{\{\widetilde{D}_1,\widetilde{D}_2\}}(\widetilde{s}_3)&=\left\{\frac{2}{3}_{\widetilde{D}_1},\frac{1}{3}_{\widetilde{D}_2}\right\},
\end{align*}
representing the three Nash equilibria.
\end{example}

\section{First mover rationality}
\label{app: first mover rationality}
In this section, we define first mover rationality, which corresponds to 
making a decision and then assuming that the other players will adapt their decisions
in response to that. In order to act like a first mover, an agent must have a belief model 
indicating which other nodes are believed to be agents and what their utility functions are believed to be.

\begin{definition}
    Let a mechanized causal graph $\mec$ and mechanism variable 
    $\widetilde{S}\in \widetilde{\rV}$ be given. A \textbf{belief model} is a tuple $(\widetilde{\rA},\bm{\mathcal{U}})$ where
    \begin{enumerate}
        \p $\widetilde{\rA}=(\widetilde{A}_1,\dots, \widetilde{A}_n)$ is a sequence of distinct nodes in $\widetilde{\rV}\backslash \{\widetilde{S}\}$.
         Intuitively, it corresponds to the set of nodes believed to be agents by $\widetilde{S}$.  
        \p $\bm{\mathcal{U}}=(\mathcal{U}_1,\dots, \mathcal{U}_n)$ is a sequence of utility functions.
        Intuitively, $\mathcal{U}_i$ is the utility believed by $\widetilde{S}$ to belong to 
        agent $A_i$. 
        \end{enumerate}
\end{definition}

Given a belief model, we can define first mover rationality as 
the decision that maximizes the expected utility given that other agents 
(according to the belief model) employ best response rationality.

\begin{definition}
We define \textbf{first mover rationality} $\mathcal{R}^{\text{FM}}_{\widetilde{S}}$ for $\widetilde{S}\in \widetilde{\rV}$ under 
belief model $(\widetilde{\rA},\bm{\mathcal{U}})$ and utility $\mathcal{U}$ by the condition that,
for $\widetilde{\rc}\in \dom (\widetilde{\rV}\backslash\{\widetilde{S}\})$,
\begin{align*}
    \widetilde{s}\in \mathcal{R}^{\text{FM}}_{\widetilde{S}}(\widetilde{\rc}) 
    \text{ if and only if }
    \widetilde{s} \in \left\{\proj_{\widetilde{S}}(\widetilde{\rv})\mid 
     \widetilde{\rv}\in \underset{\widetilde{\rw}\in \rR}{\argmax} \ \mathbb{E}_{\widetilde{\rw}}(\mathcal{U}(\rV)) \right\},
\end{align*}
where $\rR$ are those $\widetilde{\rv}$ of $\dom(\widetilde{\rV})$ for which
\begin{enumerate}
    \p $\proj_{\widetilde{\rV}\backslash \widetilde{\rA}}(\widetilde{\rv})=\proj_{\widetilde{\rV}\backslash \widetilde{\rA}}(\widetilde{\rc})$. 
    That is, non-agents do not respond by altering their distributions.
    \p For all $\widetilde{A}_i\in \widetilde{\rA}$, $\proj_{\widetilde{A}_i}(\widetilde{\rv})\in \mathcal{R}^{\text{BR}}_{\widetilde{A}_i;\mathcal{U}_i}(\proj_{\widetilde{\rV}\backslash \{\widetilde{A}_i\}}(\widetilde{\rv}))$, 
    where $\mathcal{R}^{\text{BR}}_{\widetilde{A}_i;\mathcal{U}_i}$ is the best response rationality relation for $\widetilde{A}_i$ 
    with utility $\mathcal{U}_i$.
    That is, other agents are assumed to respond $\mathcal{R}^{\text{BR}}$-rationally to their own utility functions.
\end{enumerate}
\end{definition}

In the following example we show that whether or not we have collective agency depends 
on the rationality principle employed by the agents. Concretely, we show
an example where best response rationality does not induce collective agency, whereas first mover
rationality does.

\begin{example}\textbf{Two agents with shared utility.}

    Consider a situation where we have two $\mathcal{R}^{\text{BR}}$-agents optimizing a shared utility function. 
    Signature given by 
    \begin{align*}
        \widetilde{D}_1&\in \{0,1\}\\
        \widetilde{D}_2&\in \{0,1\}\\
        \widetilde{U}&\in \mathbb{R}^{\{0,1\}\times\{0,1\}},
    \end{align*}
    and object variables
    \begin{align*}
        \mathcal{F}_{D_1}^{\widetilde{d}_1}(\mathcal{E}_{D_1})&=\widetilde{d}_1\\
        \mathcal{F}_{D_2}^{\widetilde{d}_2}(\mathcal{E}_{D_2})&=\widetilde{d}_2\\
        \mathcal{F}_{U}^{\widetilde{u}}(D_1,D_2,\mathcal{E}_U)&=\widetilde{u}(D_1,D_2).
    \end{align*}
    
    Assume that $\widetilde{D}_1$ and $\widetilde{D}_2$ are both $\mathcal{R}^{\text{BR}}$-rational with 
    respect to utility $\mathcal{U}(\rv)=\proj_U(\rv)$, that is,
    \begin{align*}
    \mathcal{F}_{\widetilde{D}_1}(\widetilde{D}_2,\widetilde{U})\in\underset{\widetilde{d}_1\in \{0,1\}}{\argmax} \  \mathbb{E}_{\widetilde{U},\widetilde{d}_1,\widetilde{D}_2}(U)\\
        \mathcal{F}_{\widetilde{D}_2}(\widetilde{D}_1,\widetilde{U})\in\underset{\widetilde{d}_2\in \{0,1\}}{\argmax} \  \mathbb{E}_{\widetilde{U},\widetilde{D}_1,\widetilde{d}_2}(U).
    \end{align*}
    Since both agents are optimizing the same utility function, we may wonder 
    if we can abstract this model into a collective agent, that is, a model with a single
    ($\mathcal{R}^{\text{BR}}$, $\mathcal{U}$) agent. 
    Consider high-level $\mec^*$ given by mechanism signature 
    \begin{align*}
        \widetilde{D}^*\in \{0,1\}^2\\
        \widetilde{U}^*\in \mathbb{R}^{\{0,1\}\times\{0,1\}}
    \end{align*}
    and object variables
    \begin{align*}
        \mathcal{F}_{D^*}^{\widetilde{d}^*}(\mathcal{E}_{D^*})&=\widetilde{d}^*\\
        \mathcal{F}_{U^*}^{\widetilde{u}^*}(D^*,\mathcal{E}_{U^*})&=\widetilde{u}^*(D^*).
    \end{align*}
    Consider the value mappings 
    \begin{align*}
        \tau_{D^*}(\{d_1,d_2\})&=(d_1,d_2)\\
        \tau_{U^*}(u)&=u
    \end{align*}
    and intervention mappings given by
    \begin{align*}
        \omega_{\widetilde{D}^*}(\{\widetilde{d}_1,\widetilde{d}_2\})&=(\widetilde{d}_1,\widetilde{d}_2)\\
        \omega_{\widetilde{U}^*}(\widetilde{u})&=\widetilde{u}.
    \end{align*}
    Assume that $\widetilde{D}^*$ is $\mathcal{R}^{BR}$-rational with respect to utility $U^*$ 
    that is, $\mathcal{U}(\rv^*)=\proj_{U^*}(\rv^*)$. 
    Consider the following mechanism for $U$:
    $$\widetilde{u}=(d_1,d_2)\mapsto \mathbbm{1}(d_1=d_2=0)+2\mathbbm{1}(d_1=d_2=1)$$
    We can convince ourselves that $\mec^*$ is not a mechanized abstraction of 
    $\mec$ since 
    $\Sol(\widetilde{\mathcal{M}};\widetilde{u})=\{\{0_{\widetilde{D}_1},0_{\widetilde{D}_2},\widetilde{u}\},\{1_{\widetilde{D}_1},1_{\widetilde{D}_2},\widetilde{u}\}\}$
     and $\Sol(\widetilde{\mathcal{M}}^*;\omega(u))=\{\{(1,1)_{\widetilde{D}^*},\widetilde{u}\}\}$, 
    have different cardinalities. 
    
    If, on the other hand $\widetilde{D}_1$ and $\widetilde{D}_2$ had employed first mover rationality with accurate belief models, then
    \begin{align*}
    \mathcal{F}_{\widetilde{D}_1}(\widetilde{D}_2,\widetilde{U})\in\underset{\widetilde{d}_1\in \{0,1\}}{\argmax}  \ \underset{\widetilde{d}_2\in \{0,1\}}{\max} \ \mathbb{E}_{\widetilde{U},\widetilde{d}_1,\widetilde{d}_2}(U)\\
        \mathcal{F}_{\widetilde{D}_2}(\widetilde{D}_1,\widetilde{U})\in\underset{\widetilde{d}_2\in \{0,1\}}{\argmax}  \ \underset{\widetilde{d}_1\in \{0,1\}}{\max} \ \mathbb{E}_{\widetilde{U},\widetilde{d}_1,\widetilde{d}_2}(U),
    \end{align*}
    and $\mec^*$ would be a strong mechanized abstraction of $\mec$, 
    suggesting that we can view $\widetilde{D}_1$ and $\widetilde{D}_2$ as a 
    collective $\mathcal{R}^{\text{BR}}$-agent.  
    \end{example}

\section{Formal description of the actor-critic example (\Cref{ex: actor-critic introduction} and 
\Cref{sec: actor-critic})}
\label{app: actor-critic details}

Formally, consider the signature\footnote{Rather than having object nodes $W$ and $Y$, representing the utilities of critic 
    and actor, respectively, we could have represented the utilities as 
    external quantities computed from the object nodes, cf. \definitionref{def: utility}. Here, we 
    choose to represent the utilities as object nodes to closely follow the 
    presentation in \citet{kenton2023discovering}.}

    \begin{align*}
    \begin{aligned}
        \widetilde{R}&\in [0,1]^2 &&  \text{Probability of reward signal (+1) given state 0 or 1.}\\
        \widetilde{W}& \in \{1\} && \text{Utility of the critic $W:=-(R-Y)^2$}\\
         \widetilde{Y}&\in \{1\} && \text{Utility of the actor $Y:=Q(A)$ }\\
        \widetilde{S}&\in [0,1]^2 && \text{Probability of state 1 given action 0 or 1.}\\
        \widetilde{Q}&\in[0,1]^2 && \text{Critic }(Q:=\widetilde{Q})\\
        \widetilde{A}&\in \{0,1\} && \text{Action ($A:=\widetilde{A}$)}
        \\
    \end{aligned}
    \end{align*}
    And structural assignments encoding that the actor and critic 
    are playing best response to $Y$ and $W$, respectively:
    \begin{align*}
        \widetilde{R}&:=(r_0,r_1)\\
        \widetilde{W}&:=1 \\
        \widetilde{Y}&:=1 \\
        \widetilde{S}&:=(s_0,s_1)\\
        \widetilde{Q}&:=
            \left(\widetilde{R}[0](1-\widetilde{S}[0])+ \widetilde{R}[1]\widetilde{S}[0],
            \widetilde{R}[0](1-\widetilde{S}[1])+ \widetilde{R}[1]\widetilde{S}[1] \right) \\
        \widetilde{A}&:=
            \mathbbm{1}(\widetilde{Q}[0]\leq \widetilde{Q}[1]).
    \end{align*}
    Since the mechanism graph is acyclic, 
    there is only one solution $\widetilde{\rs}$ to this system of equations.
    Let us consider the following abstraction of the system:
    \begin{align*}
    \begin{aligned}
        \widetilde{R}^*&\in [0,1]^2\  &&\text{Probability of reward signal (+1) given state 0 or 1}\\
        \widetilde{S}^*&\in [0,1]^2 &&\text{Probability of state 1 given action}\\
        \widetilde{A}^*&\in \{0,1\} &&\text{Action ($A:=\widetilde{A}^*$)},
    \end{aligned}
    \end{align*}
    with structural assignments
    \begin{align*}
        \widetilde{R}^*&:=(r_0,r_1)\\
        \widetilde{S}^*&:=(s_0,s_1)\\
        \widetilde{A}^*&:=\mathbbm{1}\left(
            \widetilde{S}^*[0]\widetilde{R}^*[1]+
            (1-\widetilde{S}^*[0])\widetilde{R}^*[0]\leq 
            \widetilde{S}^*[1]\widetilde{R}^*[1]+
            (1-\widetilde{S}^*[1])\widetilde{R}^*[0]\right).
    \end{align*}
    To tie the models together, we define the value mappings and intervention mappings
    \begin{align*}
        \tau_{A^*}(a)&=a & \omega_{\widetilde{A}^*}(\widetilde{a})&=\widetilde{a}\\
        \tau_{S^*}(s)&=s & \omega_{\widetilde{S}^*}(\widetilde{s})&=\widetilde{s}\\
        \tau_{R^*}(r)&=r & \omega_{\widetilde{R}^*}(\widetilde{r})&=\widetilde{r}.
    \end{align*}

\section{Proof of \propositionref{prop:no_emergence_nontrivial_agency}}
\label{app:proof_no_emergence_nontrivial_agency}

\textbf{Proposition \ref{prop:no_emergence_nontrivial_agency}.} \textit{\Paste{prop:no_emergence_nontrivial_agency}}

\begin{proof}
    Let $\widetilde{\rc}^*_1,\widetilde{\rc}^*_2\in \dom(\widetilde{\rV}^*\backslash\{\widetilde{S}^*\})$ be
    arbitrary settings. 
    Since $\mec^*$ is a strong abstraction of $\mec$, there exists low-level interventions
    $\widetilde{\ry}_1,\widetilde{\ry}_2\in \dom(\part_{\widetilde{\rV}^*\backslash \{\widetilde{S}^*\}})$
    such that $\omega_{\widetilde{\rV}^*\backslash \{\widetilde{S}^*\}}(\widetilde{\ry}_1)=\widetilde{\rc}^*_1$ and 
    $\omega_{\widetilde{\rV}^*\backslash \{\widetilde{S}^*\}}(\widetilde{\ry}_2)=\widetilde{\rc}^*_2$.  
    Since the mechanism nodes in $\part_{\widetilde{S}^*}$ have independent mechanisms,  
    $\mathbb{P}_{\Sol(\widetilde{\mathcal{M}};\widetilde{\ry}_2)}(\part_{S^*}\mid \part_{\PA_{S^*}})=
    \mathbb{P}_{\Sol(\widetilde{\mathcal{M}};\widetilde{\ry}_1)}(\part_{S^*}\mid \part_{\PA_{S^*}})$, 
    which implies that
     $\mathbb{P}_{\Sol(\widetilde{\mathcal{M}}; \widetilde{\ry}_1)}(\tau_{S^*}(\part_{S^*})\mid \tau_{\PA_{S^*}}(\part_{\PA_{S^*}}))=
     \mathbb{P}_{\Sol(\widetilde{\mathcal{M}};\widetilde{\ry}_2)}(\tau_{S^*}(\part_{S^*})\mid \tau_{\PA_{S^*}}(\part_{\PA_{S^*}}))$ by
     injectivity of $\tau_{\PA_{S^*}}$. Causal consistency 
    (\definitionref{def:constructive_abstraction}) now implies that 
    \begin{align*}
        \mathbb{P}_{\widetilde{\rc}^*_1,\mathcal{F}_{\widetilde{S}^*}(\widetilde{\rc}^*_1)}(S^*\mid \PA_{S^*})
        &= \mathbb{P}_{\Sol(\widetilde{\mathcal{M}}^*;\widetilde{\rc}^*_1)}(S^*\mid \PA_{S^*})\\
        &= \mathbb{P}_{\Sol(\widetilde{\mathcal{M}}; \widetilde{\ry}_1)}(\tau_{S^*}(\part_{S^*})\mid \tau_{\PA_{S^*}}(\part_{\PA_{S^*}}))\\
        &= \mathbb{P}_{\Sol(\widetilde{\mathcal{M}};\widetilde{\ry}_2)}(\tau_{S^*}(\part_{S^*})\mid \tau_{\PA_{S^*}}(\part_{\PA_{S^*}}))\\
        &= \mathbb{P}_{\Sol(\widetilde{\mathcal{M}}^*;\widetilde{\rc}^*_2)}(S^*\mid \PA_{S^*})\\
        &=\mathbb{P}_{\widetilde{\rc}^*_2,\mathcal{F}_{\widetilde{S}^*}(\widetilde{\rc}^*_2)}(S^*\mid \PA_{S^*})
    \end{align*}
    Since $\widetilde{\rc}^*_1,\widetilde{\rc}^*_2\in \dom(\widetilde{\rV}^*\backslash\{\widetilde{S}^*\})$ were arbitrary, 
    we have that $\widetilde{S}^*$ is not a non-trivial agent.
\end{proof}

\section{Proof of \propositionref{prop:actor_critic_abstraction}}
\label{app:proof_actor_critic}

\textbf{Proposition \ref{prop:actor_critic_abstraction}.} \textit{\Paste{prop:actor_critic_abstraction}}

\begin{proof}

To show that $\mec^*$ is a strong
 mechanized abstraction, we need to show that
  (1) consistency holds for all low-level interventions where $\omega$ is defined,
 and (2) $\omega_{\widetilde{V}^*}$ is surjective onto 
 $\dom(\widetilde{V}^*)$ for all $\widetilde{V}^*\in \widetilde{\rV}^*$.

First, consider surjectivity.
Since the intervention mappings are 
the identity functions and the ranges of the relevant low-level and high-level variables
 are identical,
 we have surjectivity.

Second, we check consistency.
Let $\widetilde{\ry} = \{(s_0,s_1)_{\widetilde{S}},(r_0,r_1)_{\widetilde{R}}\}$ 
be an arbitrary low-level intervention 
on $\part_{\widetilde{S}^*}\cup \part_{\widetilde{R}^*}=\{\widetilde{R},\widetilde{S}\}$. 
Then $\omega(\widetilde{\ry}) = \{(s_0,s_1)_{\widetilde{S}^*},(r_0,r_1)_{\widetilde{R}^*}\}$ 
and
\begin{align*}
    \Sol(\widetilde{\mathcal{M}}^*; \omega(\widetilde{\ry})) = 
    \left\{\left\{(s_0,s_1)_{\widetilde{S}^*},(r_0,r_1)_{\widetilde{R}^*},\mathbbm{1}\left(s_0r_1+(1-s_0)r_0\leq
        s_1r_1+(1-s_1)r_0\right)_{\widetilde{A}^*}\right\}\right\}
\end{align*}
In the low-level model $\mec$:
\begin{align*}
    \Sol(\widetilde{\mathcal{M}}; \widetilde{\ry})=
    \Big\{\Big\{&(r_0,r_1)_{\widetilde{R}}, 1_{\widetilde{Y}}, 1_{\widetilde{W}}, (s_0,s_1)_{\widetilde{S}},\\
    &\left(r_0(1-s_0)+ r_1s_0, r_0(1-s_1)+ r_1s_1 \right)_{\widetilde{Q}},\\
    &\mathbbm{1}\left(s_0r_1+(1-s_0)r_0\leq s_1r_1+(1-s_1)r_0\right)_{\widetilde{A}} \Big\}\Big\}
\end{align*}

Since $A^*$ and $A$ are equal to 0 and 1 for the same values of $r_0,r_1,s_0,s_1$, 
we have that 
$$\mathbb{P}_{\Sol(\widetilde{\mathcal{M}};\widetilde{\ry})}(\tau(\rV))=\mathbb{P}_{\Sol(\widetilde{\mathcal{M}}^*; \omega(\widetilde{\ry}))}(\rV^*)$$

A similar calculation shows consistency for interventions on other
  elements in $\{\bigcup_{V^*\in \widetilde{\rY}^*} 
  \part_{\widetilde{V}^*}\mid \widetilde{\rY}^*\subseteq \widetilde{\rV}^*\}$.

We want to argue that in the abstracted model, $\widetilde{A}^*$ is $\mathcal{R}^{\text{BR}}$ rational 
under utility function $\mathcal{U}(\rv^*)
=\proj_{R^*}(\rv^*)$. So we must argue that $\widetilde{A}^*$ responds $\mathcal{R}^{\text{BR}}$-rationally to $\mathcal{U}$.
 Consider context $\widetilde{\rc}=\{(s_0,s_1)_{\widetilde{S}^*},(r_0,r_1)_{\widetilde{R}^*}\}$
  and the 
 responses $0_{\widetilde{A}^*}$ or $1_{\widetilde{A}^*}$:
\begin{align*}
    \mathbb{E}_{\{0_{\widetilde{A}^*},(s_0,s_1)_{\widetilde{S}^*},(r_0,r_1)_{\widetilde{R}^*}\}}(R^*)&=s_0r_1+(1-s_0)r_0\\
    \mathbb{E}_{\{1_{\widetilde{A}^*},(s_0,s_1)_{\widetilde{S}^*},(r_0,r_1)_{\widetilde{R}^*}\}}(R^*)&=s_1r_1+(1-s_1)r_0
\end{align*}
So 
\begin{align*}
    \mathcal{R}^{\text{BR}}_{\widetilde{A}^*}(\widetilde{\rc})=\begin{cases}
        \{1_{\widetilde{A}^*}\} & s_1r_1+(1-s_1)r_0>s_0r_1+(1-s_0)r_0\\
        \{0_{\widetilde{A}^*}\} & s_1r_1+(1-s_1)r_0<s_0r_1+(1-s_0)r_0\\
        \{0_{\widetilde{A}^*},1_{\widetilde{A}^*}\} & s_1r_1+(1-s_1)r_0=s_0r_1+(1-s_0)r_0,
    \end{cases}
\end{align*}
and we conclude that
 $\mathcal{F}_{\widetilde{A}^*}(\widetilde{\rc})\in \mathcal{R}^{\text{BR}}_{\widetilde{A}^*}(\widetilde{\rc})$ 
 for every $\widetilde{\rc}\in \dom(\widetilde{\rV}^*\backslash\{\widetilde{A}^*\})$.
\end{proof}

\section{Further details on \Cref{sec:surrogate_modeling}}

The code to reproduce all results is available at \href{https://github.com/FrederikHJ/collective_agency_clear/}{this link}. It runs within a few minutes on a standard laptop.

\subsection{Voting mechanisms}\label{app: details voting mechanisms}

We consider three different voting mechanisms: Vickrey-Clarke-Groves (VCG), Median Voting, and Random Dictator.

\paragraph{VCG Voting.}
Under VCG voting, each citizen reports the parameters 
of their utility function and monetary transfers 
ensure that citizens are incentivized to report their true parameters
\citep{tadelis2013game}: $\hat{a}_{ci}=(a_{ci}-\lambda_{ci}),\hat{b}_{ci}=b_{ci},\hat{d}_{ci}=d_{ci}$. 
The mechanism then chooses the social optimum within each country, i.e.
$$q_c^{\text{NE}}(\bm{\lambda})=\underset{q_c\in \mathbb{R}}{\argmax} \sum_i \left(\hat{a}_{ci}q_c-\hat{b}_{ci}q_c^2-\hat{d}_{ci}Q_W^2\right).$$
While this 
voting mechanism is unrealistic,
it serves as a simple baseline example: since this sum has the same form as \Cref{eq: high level utility}, we can calculate the low-level NE analytically using \Cref{eq: ne}.
This also means that, for the VCG mechanism, we know that the high-level model is an exact mechanized abstraction of the low-level model, and we can compare the estimated parameters to ground-truth parameters.

\paragraph{Median Voting.}
For Median Voting, we cannot calculate the low-level NE analytically.
Therefore, we approximate it using iterative best response.
In particular, we initialize all pollution levels to 0 and then iteratively
update the response of each country using the median vote within that country (using a damping factor of 0.3).
We do this until convergence (tolerance of 1e-6 in the 2-norm).

\paragraph{Random Dictator.}
For Random Dictator voting, we randomly select a dictator from each country and 
use the dictators' utility functions to derive the low-level Nash equilibrium.

\subsection{Further details on experiments}\label{app: details of voting experiment}

\textbf{Population structure.}
We consider a population of 5 countries and 1000 citizens in total. The sizes are sampled from a log-normal distribution and scaled to sum to 1000. 
The countries and citizens are fixed throughout the experiments. 

\textbf{Citizen parameters.} 
The citizen parameters are sampled i.i.d. uniformly within the ranges 
\begin{align*}
    a_{ci}\in [0.35,0.65],\\
    b_{ci}\in [7, 13]/N_c,\\
    d_{ci}\in [0.05, 0.15]/C,
\end{align*}
where $N_c$ is the size of country $c$, and $C$ is the number of countries. The parameter ranges
are manually chosen to be somewhat reasonable.

\textbf{Interventions $\lambda$.}
We sample different intervention means for each country. The means are sampled uniformly within the 
range $[0.0, 0.1]$. Given a mean $\lambda_c$, we sample interventions on individual 
citizens from the beta distribution with mean $10\lambda_c$ 
and concentration parameter 10. We then scale the sampled values back to the range $[0.0, 0.1]$. 

\textbf{$\delta_c$ estimation.}
For each country $c$, we sample 10 interventions where 
$(\lambda_{ci})_c=\bm{0}$, that is, no intervention on country $c$.
We estimate $\delta_c$ by linear regression of $q_c$ on $Q_W$ (which is justified by \Cref{eq: ne}).
For random dictator voting, instead of estimating $\delta_c$, we plug in the sum 
$\frac{1}{N_c}\sum_{i} \frac{d_{ci}}{b_{ci}}$ as an estimate (the mean of citizen-level $\delta_{ci}=d_{ci}/b_{ci}$).
We do this to give the model a better chance of learning a useful $\omega$. 
Estimating $\delta_c$ using the regression methods introduces bias since there is positive
correlation between $q_c$ and $Q_W$ introduced by the stochasticity of picking the dictator. 
Alternatively, we could estimate $\delta_c$ using several regressions, each with a fixed dictator,
and then take an average.

\textbf{Surrogate model architecture.}
We use a fully connected neural network with 4 hidden layers of 128, 256, 256, and 128 neurons respectively.
We use the ReLU activation function for all layers.
The input is the intervention vector $(\lambda_{ci})_{ci}$ for all citizens
and the output is $\alpha_c$ for each country.

\textbf{Training procedure.}
We train the model on 1,000 intervention samples for 100 epochs using
the Adam optimizer with a learning rate of $10^{-3}$ and a batch size of 32. 
The results are evaluated on 500 held-out test interventions.

\subsection{Further details on results}\label{app: voting experiment results}
\phantom{dd}\\
\begin{table}[h!]
    \centering
    \begin{tabular}{c|c|c|c|c|c}
        Country & Citizens & MAE($\delta_c$) & MAE($\alpha_c$) & MAE($q_c$) & Baseline MAE($q_c$) \\ \hline
        0 & 180 & 0.000 & 0.025 & 0.011 & 0.233 \\
        1 & 107 & 0.000 & 0.029 & 0.007 & 0.159 \\
        2 & 202 & 0.000 & 0.065 & 0.017 & 0.256 \\
        3 & 412 & 0.000 & 0.107 & 0.022 & 0.468 \\
        4 & 99 & 0.000 & 0.025 & 0.005 & 0.138 \\
    \end{tabular}
    \caption{Results for VCG mechanism. We report the Mean Absolute Error (MAE) 
    for the parameters $\delta_c$ and $\alpha_c$, and for the pollution outcome $q_c$.
		We also include the baseline MAE for $q_c$, which is the MAE if 
        we constantly predict the ground truth low-level Nash equilibrium at $\bm{\lambda = 0}$.
		The model is $\approx$ 95\% better than the baseline, see \Cref{table: mechanism comparison}.
		For VCG we can recover
         $\delta_c$ perfectly using \Cref{eq: ne} and two 
         distinct $(q_c^{\text{NE}},Q_W^{\text{NE}})$ pairs.}
    \label{table: vcg results}
\end{table}

\begin{table}[h]
    \centering
    \begin{tabular}{c|c|c|c}
        Country & Citizens & MAE($q_c$) & Baseline MAE($q_c$) \\ \hline
        0 & 180 & 0.029 & 0.235 \\
        1 & 107 & 0.032 & 0.144 \\
        2 & 202 & 0.025 & 0.259 \\
        3 & 412 & 0.035 & 0.479 \\
        4 & 99 & 0.021 & 0.131 \\
    \end{tabular}
    \caption{Results for Median Voting mechanism.
     In contrast to VCG, we have no ground truth parameters $\alpha_c$ and $\delta_c$ 
     to compare with. However, the surrogate model still accurately predicts the pollution outcome
      $q_c$,
       suggesting that a collective agency framework can effectively model median voting mechanisms.}
    \label{table: median results}
\end{table}

\begin{table}[h]
    \centering
    \begin{tabular}{c|c|c|c}
        Country & Citizens & MAE($q_c$) & Baseline MAE($q_c$) \\ \hline
        0 & 180 & 0.949 & 0.878 \\
        1 & 107 & 0.530 & 0.525 \\
        2 & 202 & 1.163 & 1.042 \\
        3 & 412 & 1.916 & 1.694 \\
        4 & 99 & 0.519 & 0.506 \\
    \end{tabular}
    \caption{Results for Random Dictator mechanism. 
    We use $\delta_c=\frac{1}{N_c}\sum_i \frac{d_{ci}}{b_{ci}}$. 
    The baseline is the average low-level random dictator Nash equilibrium at $\lambda=0$ (averaged over dictator draws).
    The model performs worse than the baseline overall (-9.3\%), see \Cref{table: mechanism comparison}.}
    \label{table: random dictator results}
\end{table}

\section{Applications to MARL and LLM Agents}
\label{app:applications}

In this appendix, we discuss how our theoretical framework for collective agency 
may be applied to multi-agent reinforcement learning (MARL) and large language 
model (LLM) agents, and how it connects to AI safety. While a full treatment of 
these applications is beyond the scope of this work, we hope that the discussion 
below provides useful starting points for future research.

\subsection{Applications to MARL}

There are several subtleties that arise when attempting to translate insights from our framework in to the MARL setting. Here we provide a set of initial suggestions and observations that we hope makes the practical use of our formalism to MARL settings more straightforward. Such settings may provide a natural testbed for experimentation prior to developing real-world applications.

\paragraph{Approximate Equilibria and Bounded Rationality.}
A key feature of our framework is that it is agnostic to the specific notion 
of rationality employed by the agents. The rationality relation $\mathcal{R}$ 
(see \definitionref{def: best response rationality} and 
\Cref{app: first mover rationality}) can encode a wide range of behavioral 
assumptions: exact best response, $\varepsilon$-best response, gradient-based 
learning rules, or any other decision procedure. This flexibility means that 
the framework applies even when agents are only approximately rational, as is 
typical in MARL, where agents may use gradient-based policy optimization and 
converge only to approximate equilibria. In such settings, one could define a 
rationality relation that captures $\varepsilon$-optimality or local optimality, 
and then ask whether a collective agent model under this relaxed rationality 
relation constitutes a valid mechanized abstraction of the low-level system.

\paragraph{Learning Dynamics and Non-Stationarity.}
The mechanized causal games 
studied in this paper can be thought of as modeling
 equilibria: the solutions $\Sol(\widetilde{\mathcal{M}})$ 
 correspond to settings of the mechanism variables that 
 are mutually consistent given the rationality relations.
  On the one hand, this is useful because our framework applies regardless 
  of the algorithm that agents use to reach an equilibrium (whether they use 
  gradient-based learning, linear programming, or some other method). 
  On the other hand, the models in this paper do not naturally support 
  analysis of the learning dynamics themselves, or the non-stationarity 
  that arises when agents are simultaneously updating their strategies. 
  In particular, MARL algorithms can produce cycles in strategy space 
  \citep{Zinkevich2005}, which are not captured by equilibrium-based models. 
  Understanding the relationship between causal models and dynamical systems 
  more generally is an active area of research; see, e.g., 
  \citet{Reisach2025} for a discussion of complications and ambiguities 
  related to the relationship between causal models and time, and 
  \citet{jorgensen2025causal} for ambiguities related to interventions.

\paragraph{Empirical Game-Theoretic Analysis.}
One promising avenue for applying our framework to large or complex MARL 
settings is through \emph{empirical game-theoretic analysis} 
\citep{Wellman2025}. In this approach, overall policies (which might be 
large neural networks in practice) are modeled as pure strategies in a 
`meta-game', which is then iteratively expanded and refined. By constructing 
a mechanized causal game at the level of the meta-game, it becomes possible 
to tractably analyze complex MARL settings using the tools from game theory 
and causal abstraction that we develop in this paper.

\subsection{Applications to LLM-Based Multi-Agent Systems}

Multi-agent systems composed of LLM-based agents are becoming increasingly 
common, from debate protocols \citep{Irving2018,Du2024,Li2025} 
to multi-agent problem-solving and planning systems.
Our framework provides a principled way to ask: when does a group of LLM 
agents, each with its own prompt, context, and objectives, act as if it 
were a single collective agent pursuing a coherent goal?

\paragraph{Collusion in Debate and AI Oversight.}
One possible application is to debate-based AI safety protocols 
\citep{Irving2018}. The effectiveness of AI safety via debate relies on the 
debating agents being adversarial rather than collectively pursuing a shared 
goal. If the debating agents were to collude -- effectively forming a 
collective agent that optimizes for some objective other than truthful 
argumentation -- the safety guarantees of the debate protocol could break 
down. 
Similarly, in other settings where safety arises from some kind of automated oversight or adversarial relationship, cooperation is often undesirable.
Our framework allows one to rule out collective agency under 
certain conditions (\propositionref{prop:no_emergence_nontrivial_agency} 
provides a simple example), which could be used to provide assurance 
that mechanisms like debate are not illegibly pursuing unintended goals.

\paragraph{Emergent Goals in Networks of LLM Agents.}
More generally, whenever multiple LLM agents interact -- whether through 
explicit communication channels, shared memory, or indirect coordination 
via a shared environment -- our framework can be used to assess whether 
emergent collective agency is present and, if so, what objective the 
collective agent appears to be optimizing.
Recent work on measuring goal-directedness in individual LLM agents 
\citep{Xu2024,Everitt2025} provides complementary tools that could be 
combined with our collective agency framework to detect and characterize 
emergent group-level goals in multi-agent LLM systems.

\end{document}